\documentclass{article}
\usepackage[utf8]{inputenc} 
\usepackage[T1]{fontenc}    
\usepackage{hyperref}       
\usepackage{url}            
\usepackage{booktabs}       
\usepackage{amsfonts}       
\usepackage{nicefrac}       
\usepackage{microtype}      
\usepackage{lipsum}

\usepackage{amssymb, amsmath, latexsym, amsthm}
\usepackage{comment}
\usepackage{graphicx}
\usepackage{authblk}
\usepackage{geometry}

\def\c{{\bm c}}

\def\e{{\bm e}}

\def\x{{\bm x}}
\def\y{{\bm y}}

\def\0{{\bm 0}}
\def\1{{\bm 1}}

\def\Rbb{\mathbb{R}}

\def\Nbb{\mathbb{N}}
\def\sign{{\mathrm{sign}}}

\newtheorem{thm}{Theorem}

\newtheorem{lem}{Lemma}
\newtheorem{assmp}{Assumption}

\makeatletter
\newcommand{\figcaption}[1]{\def\@captype{figure}\caption{#1}}
\newcommand{\tblcaption}[1]{\def\@captype{table}\caption{#1}}
\makeatother

\RequirePackage{algorithm}
\RequirePackage{algorithmic}
\usepackage{bm}
\def\z{\bm{z}}
\def\<{\langle}
\def\>{\rangle}
\def\bxi{\bm{\xi}}
\def\balpha{{\bm{\alpha}}}
\usepackage{color}

\title{Variable Selection for Nonparametric Learning \\ with Power Series Kernels}

\author[1]{Kota Matsui\thanks{kota.matsui@riken.jp}}
\author[1]{Wataru Kumagai}
\author[2]{Kenta Kanamori}
\author[3]{Mitsuaki Nishikimi}
\author[4,1]{Takafumi Kanamori}

\affil[1]{Center for Advanced Intelligence Project, RIKEN}
\affil[2]{Department of Computer Science, Nagoya Institute of Technology}
\affil[3]{Department of Emergency and Critical Care, Nagoya University}
\affil[4]{Department of Mathmatical and Computing Science, Tokyo Institute of Technology}

\begin{document}

\maketitle

\begin{abstract}
  In this paper, we propose a variable selection method for general nonparametric kernel-based estimation. 
The proposed method consists of two-stage estimation: (1) construct a consistent estimator of 
the target function,  
(2) approximate the estimator using a few variables by $\ell_1$-type penalized estimation.
 We see that the proposed method can be applied to various kernel nonparametric 
estimation such as kernel ridge regression, kernel-based density and density-ratio estimation. 
We prove that the proposed method has the property of the variable selection consistency 
when the power series kernel is used. 
Here the power series kernel is a certain class of kernels containing polynomial kernel and 
exponential kernel.
This result is regarded as an extension of the variable selection consistency
for the non-negative garrote, which is a special case of the adaptive lasso,
to the kernel-based estimators. 
Several experiments including simulation studies and 
real data applications show the effectiveness of the proposed method. 
\end{abstract}




\section{Introduction}
\label{intro}

Variable selection is quite important in various machine learning tasks, to improve the performance, 
select more cost-effective subset of the features and guarantee the interpretability~\cite{guyon2003introduction}. 
Many variable selection methods has been developed for parametric learning including linear models. 
On the other hand, it is desirable to be able to do variable selection in nonparametric learning 
although such study has been done only in very limited problems. 
As a motivating example, in the clinical prognostic problem at multiple 
hospitals~\cite{nishikimi2017novel}, learning of prediction models that will be used at target hospital is performed 
using integrated data from multiple hospitals.
However, since the distribution of covariates among hospitals can differ,  
adaptation
by importance weighting using density-ratio~\cite{sugiyama2012density} is required to appropriately 
learn a model. 
Although the density-ratio estimation can be done accurately by the kernel method~\cite{kanamori12:_statis}, 
all covariates must be observed at all hospitals. 
Since this is very expensive to use, it is natural that we want to select the variables which contribute to the 
density-ratio in advance.
In this paper, we propose a general variable selection method for kernel-based nonparametric learning.

In this paper, we consider kernel methods not only for regression problems but also
density ratio estimation, density estimation, density-ridge estimation, etc. 
In the kernel methods, we employ the power series kernel~\cite{zwicknagl09:_power_series_kernel}. 
We show that the power series kernel has a desirable property, i.e., the invariant property under the variable
scaling on the corresponding reproducing kernel Hilbert space. 
We prove the variable selection consistency of the kernel methods using the power series kernel and NNG
under milder assumptions than \cite{zaili16:_flexib_variab_selec_recov_spars} in which 
a kernel variant of the irrepresentable condition is assumed. 
Our result is regarded as an extension of the variable selection consistency for the original
NNG~\cite[Corollary 2]{zou2006adaptive} or adaptive lasso to the kernel-based estimators. 

The rest of the paper is organized as follows. In Section~\ref{Problem_Setup}, we explain the problem setup and 
give some definitions. Several kernel nonparametric learning problems are formulated in an unified way. 
Section~\ref{main} provides the explanation of the proposed method. 
Section~\ref{theory} is devoted to the main results. The variable selection consistency is shown in this section. 
In Section~\ref{experiment}, we show the experimental results for both synthetic and real datasets. 
Finally in Section~\ref{conclusion}, we conclude the paper with the discussion on future works.

\subsection{Related Work}
The conventional approach to the variable selection was the information criterion such as AIC or the
sequential test called forward and backward step-wise selection. 
For high-dimensional models, however, the information criterion does not work because of the computational issue, 
i.e., the combinatorial complexity appears in the choice of variables. 
Also, the statistical test needs to repeat the computation of sample statistics many times.
For high-dimensional linear models, \cite{tibshirani1996regression} proposed the so-called Lasso estimator, 
in which the $\ell_1$-norm of the coefficients was incorporated into the squared loss. 

Also, the non-negative garrote (NNG) was proposed by \cite{breiman1995better} as a modification
of the standard least square (LS) estimator for linear models.
In the NNG, each coefficient of the LS estimator is shrunk towards zero,
and its intensity is controlled by the $\ell_1$-regularization. 
The variable selection consistency of Lasso was proved under the so-called irrepresentable
condition~\cite{peng06:_model_selec_consis_lasso,Wainwright:2009:STH:1669487.1669506}. 
On the other hand, \cite{zou2006adaptive} and \cite{yuan07} proved that the NNG has the variable selection consistency 
without the irrepresentable condition.
For some kernel-based estimators, 
\cite{allen13:_autom_featur_selec_weigh_kernel_regul}
and
\cite{grandvalet03:_adapt_scalin_featur_selec_svms}
employed the NNG as the adaptive scaling for the variable selection. 
Also, \cite{zaili16:_flexib_variab_selec_recov_spars} proved the variable selection consistency
of the scaled kernel-ridge estimator under a variant of the irrepresentable condition. 

Rosasco et al.~\cite{rosasco2013nonparametric} considered the variable selection problem 
in the kernel ridge regressions. 
Instead of the scaling parameters, the authors proposed the regularization based on the 
derivatives. 
They proved only the selected variables include the target variables with high probability. 
However, it is not clear whether the extra variables can be removed with high probability. 
As the author of~\cite{rosasco2013nonparametric} pointed out, the variable selection 
consistency is not completely proved, and that was postponed to the future work. 

Several literature deal with similar but different problem setting \cite{feng2016kernelized, roth2004generalized, salzo2017solving, wang2007kernel}. 
In these papers, the sparsity is incorporated into the coefficients in the linear sum 
of the kernel functions for the kernel learning. Hence, the variable selection concerning the
covariates was out of the scope of these papers.

\section{Problem Setup}
\label{Problem_Setup}

We briefly introduce kernel methods. See \cite{berlinet2003reproducing,steinwart2008support} for details. 
Let $\mathcal{Z}$ be the domain of the $d$-dimensional data
and $k:\mathcal{Z}\times\mathcal{Z}\rightarrow\Rbb$ be a kernel function.
The reproducing kernel Hilbert space (RKHS) associated with $k$ is denoted as $\mathcal{H}$ or $\mathcal{H}_k$. 
The RKHS is a linear space consisting of real-valued functions on $\mathcal{Z}$
which is used as the statistical model. 
The inner product and the norm on $\mathcal{H}$ are represented as $\<f,g\>$ and $\|f\|=\sqrt{\<f,f\>}$ for
$f,g\in\mathcal{H}$.

In many learning algorithms, the target function is estimated by minimizing an empirical loss function $\widehat{L}(f)$  with a regularization term $\lambda R(f)$ as follows:
\begin{align}
\label{eqn:general-kernel-based-estimator}
\min_{f\in\mathcal{H}}\widehat{L}(f)+\lambda R(f), 
\end{align}
where $\widehat{L}(f)$ depends on training data and $\lambda$ is a positive parameter controlling the capacity of the statistical model.
Throughout the paper, we assume that
$R(f)=\frac{1}{2}\|f\|^2$ and
the empirical loss is expressed as the quadratic form, 
\begin{align}
  \label{eqn:quadratic-loss}
 \widehat{L}(f)=\frac{1}{2}\<f,\widehat{C}f\>-\<\widehat{g},f\>, 
 \end{align}
where a linear operator $\widehat{C}:\mathcal{H}\rightarrow\mathcal{H}$ and
an element $\widehat{g}\in\mathcal{H}$ depend on training data.
Several examples are shown below. 
Suppose that $\widehat{L}(f)$ converges to $L(f)$ in probability for each $f\in\mathcal{H}$ as the sample size 
goes to infinity. We assume that $L(f)$ has the form of 
\begin{align}
  \label{eqn:expected-quadratic-loss}
 L(f)= \frac{1}{2}\<f,Cf\>-\<g,f\>, 
\end{align}
where $C:\mathcal{H}\rightarrow\mathcal{H}$ and $g\in\mathcal{H}$ will depend on the probability distribution of the training data. 
Suppose that the target function $f^*\in\mathcal{H}$ is the minimum solution of \eqref{eqn:expected-quadratic-loss}. 

In this paper we focus on the variable selection problem for kernel methods. 
The kernel-based estimator $\widehat{f}(\x)$ usually depends on all variables in $\x=(x_1,\ldots,x_d)$. 
However, the target function $f^*$ may depend on only a few variables. The goal of the variable selection is to detect these variables. 


\subsection{Kernel-Ridge estimator}
\label{subsection:Kernel-ridge_Estimator}
In regression problems, i.i.d. samples $(\x_1,y_1),\ldots,(\x_n,y_n)\sim p(\x,y)=p(y|\x)p(\x)$ are assumed
to be generated from the model $y_i=f^*(\x_i)+\varepsilon_i$, where $f^*$ is the true regression function 
and $\varepsilon_i$ is i.i.d. random observation noise. 
The variable selection is important in order to model the relationship between $\x$ and $y$. 

The RKHS $\mathcal{H}$ is commonly used as the statistical model. 
The expected squared loss and empirical squared loss for $f\in\mathcal{H}$ are respectively defined as 
$\frac{1}{2}\int (y-f(\x))^2 p(\x,y) d{\x}dy$ and
$\frac{1}{2n}\sum_{i=1}^{n}(y_i-f(\x_i))^2$.
Let us define $C=\int k(\x,\cdot)\otimes{k(\x,\cdot)}p(\x)d\x$ and $g=\int y k(\x,\cdot) p(\x,y)d{\x}dy\in\mathcal{H}$. 
Here, $\otimes$ denotes the tensor product defined as $\<g\otimes{h}, f\>=\<h,f\>g$ for $f,g,h\in\mathcal{H}$.
Then, the loss function $L(f)=\frac{1}{2}\<f,Cf\>-\<g,f\>$ equals the expected squared  loss up to constant terms.
Likewise, $\widehat{L}(f)$ is expressed as $\frac{1}{2}\<f,\widehat{C}f\>-\<\widehat{g},f\>$ 
using $\widehat{C}=\frac{1}{n}\sum_{i=1}^{n}k(\x_i,\cdot)\otimes{k(\x_i,\cdot)}$ 
and $\widehat{g}=\frac{1}{n}\sum_{i=1}^{n}y_i k(\x_i,\cdot)\in\mathcal{H}$.
The kernel-ridge estimator $\widehat{f}$ is given as the minimizer of 
$\widehat{L}(f)+\frac{\lambda}{2}\|f\|^2$ subject to $f\in\mathcal{H}$. 
It is well-known that $\|\widehat{f}-f^*\|=o_P(1)$ holds when the output $y$ is bounded and the regularization parameter $\lambda$
is appropriately chosen;  the detail is found in the proof of Theorem~4 in~\cite{Caponnetto:2007:ORR:1290530.1290534}. 

\subsection{Kernel-based density-ratio estimator}
\label{density-ratio-est}

Density ratio is defined as the ratio of two probability densities~\cite{sugiyama2012density}. 
The density ratio is an important versatile tool in statistics and machine learning, 
since it appears in many learning problems including regression problems under the covariate-shift, two-sample
test, outlier detection, etc.
Suppose that we have training samples, $\x_1,\ldots,\x_{n_1}\sim_{i.i.d.} p$ and 
$\y_1,\ldots,\y_{n_2}\sim_{i.i.d.} q$, where $p$ and $q$ are probability densities on the domain $\mathcal{Z}$. 
Our goal is to estimate $f^*(\z)=q(\z)/p(\z)$. 

Let us consider the variable selection of the density ratio. 
For the vector $\z=(\z_a, \z_b)$, $p(\z)$ and $q(\z)$ are decomposed into conditional probabilities and
marginal ones such as $p(\z_a|\z_b) p(\z_b)$ and $q(\z_a|\z_b) q(\z_b)$. 
When $p(\z_a|\z_b)=q(\z_a|\z_b)$ holds, $q(\z)/p(\z)$ is reduced to $q(\z_b)/p(\z_b)$. 
The variable selection of the density ratio is closely related to the 
identification of the conditional probability that $p$ and $q$ have in common. 

A kernel-based density-ratio estimator called KuLSIF was proposed in~\cite{kanamori12:_statis}. 
The empirical loss of $f$ is defined by $\widehat{L}(f)=\frac{1}{2n_1}\sum_{i=1}^{n_1}f(\x_i)^2-\frac{1}{n_2}\sum_{j=1}^{n_2}f(\y_j)$. 
As $n=\min\{n_1,\,n_2\}$ tends to infinity, the empirical loss converges to 
$L(f)=\int_{\mathcal{Z}} \{\frac{1}{2}f(\z)^2p(\z)-f(\z) q(\z)\}d\z$ due to the law of large numbers. 
The minimizer of $L(f)$ is nothing but $f^*=q/p$.
The empirical quadratic loss is expressed by 
$\widehat{C}=\frac{1}{n_1}\sum_{i=1}^{n_1}k(\x_i,\cdot)\otimes k(\x_i,\cdot)$ and 
$\widehat{g}=\frac{1}{n_2}\sum_{j=1}^{n_2}k(\y_j,\cdot)$, and 
the expected loss $L(f)$ by $C=\int\! k(\x,\cdot)\otimes k(\x,\cdot) p(\x) d\x$ and
$g=\int\! k(\y,\cdot)q(\y)d\y$ up to constant terms.
\cite{kanamori12:_statis} proved the statistical consistency of the KuLSIF estimator in the $L_2$ norm.
In Appendix~\ref{Consistency_KuLSIF}, 
we prove the statistical consistency in the RKHS norm.


\subsection{Kernel-based density estimator and density-ridge estimator}

The kernel-based density estimator using infinite dimensional exponential models has been studied
in~\cite{JMLR:v18:16-011}. The problem is to estimate the function $f^*$ of the probability density 
$p(\z)\propto\exp(f^*(\z))$ using the i.i.d. samples $\z_1,\ldots,\z_n\in\mathcal{Z}$ 
from $p(\z)$. The variable selection of the probability density on the bounded domain $\mathcal{Z}$ 
is nothing but the identification of variables whose marginal probability is the uniform distribution on the domain. 

The estimator $\widehat{f}$ of $f^*$ is obtained by minimizing the Hyv\"{a}rinen score with the regularization. 
Let $\partial_a f$ be the derivative of the function $f(\z)$ w.r.t. $z_a$, i.e., 
$\frac{\partial{f}}{\partial{z}_a}(\z)$. Likewise, $\partial_{aa}f$ denotes the second-order derivative of $k$ w.r.t. the $a$-th argument. 
The empirical loss $\widehat{L}(f)$ for the Hyv\"{a}rinen score is defined from 
$\widehat{C}=\frac{1}{n}\sum_{i=1}^{n}\sum_{a=1}^{d}\partial_{a}k(\z_i,\cdot)\otimes \partial_{a}k(\z_i,\cdot)$ and 
$\widehat{g}=-\frac{1}{n}\sum_{i=1}^{n}\sum_{a=1}^{d}\partial_{aa}k(\z_i,\cdot)$.
The reproducing property for the derivative, $\partial_{a}f(\z)=\<f,\partial_{a}k(\z,\cdot)\>$,
is useful to conduct the calculation of the operator $\widehat{C}$. 
Note that $\partial_{a}k(\z,\cdot)\in\mathcal{H}$ holds for the kernel function $k$~\cite{zhou08:_deriv}. 
Likewise, the expected quadratic loss $L(f)$ is expressed by 
$C=\int\sum_{a=1}^{d}\partial_{a}k(\z,\cdot)\otimes\partial_{a} k(\z,\cdot)p(\z)d\z$ and 
$g=-\int\sum_{a=1}^{d}\partial_{aa}k(\z,\cdot)\,p(\z)d\z\in\mathcal{H}$. 
When a proper boundary condition is assumed,
the expected quadratic loss derived from $C$ and $g$ equals
$\frac{1}{2}\int \sum_{a=1}^d|\partial_{a}f(\z)-\partial_{a}f^*(\z)|^2 p(\z) d\z$ 
up to constant terms. Thus, the Hyv\"{a}rinen score is regarded as the mean square error for derivatives. 
The estimator using the Hyv\"{a}rinen score and an appropriate regularization parameter $\lambda$
has the statistical consistency $\|\widehat{f}-f^*\|=o_P(1)$ under 
a mild condition~\cite{JMLR:v18:16-011}. 

The density-ridge estimator is related to the above density estimator. 
The target is to estimate $f^*(\z)=\frac{\partial_{I}p(\z)}{p(\z)}$ using i.i.d. samples from $p(\z)$, where
$\partial_I$ is the differential operator  
$\frac{\partial^{k}}{\partial{z_{i_1}},\ldots,\partial{z_{i_k}}}$ with the set (or multiset) of non-negative
integers $I=\{i_1,\ldots,i_k\}$. 
The above density estimator corresponds to the case that the set $I$ is a singleton.
The ingredients of the quadratic loss, $\widehat{C},\widehat{g}, C$ and $g$, are defined
from the derivative of the kernel function, $\partial_{I}k(\z,\cdot)$. The estimated function 
$\widehat{f}$ is used to extract the ``ridge structure'' of the probability density 
$p(\z)$ that usually has a complex low-dimensional structure. 
The variable selection for the density-ridge estimator is important
to boost the estimation accuracy of the ridge and to reduce the computational cost. 


\section{Variable Selection using Adaptive Scaling with Power Series Kernels}
\label{main}

\subsection{Kernel Methods with Adaptive Scaling}
For variable selection, we incorporate adaptive scaling parameters to the variables in the RKHS model. 
As the adaptive scaling parameter, we employ Breiman's Non-Negative Garrote (NNG)~\cite{breiman1995better}. 
The original NNG is used to estimate the sparse vector $\bm{\beta}$ of the linear regression model
$y=\x^T{\bm\beta}+\epsilon$, where $\epsilon$ is the observation error. 
The least mean square estimator $\widehat{{\bm\beta}}_0$ is mapped to 
${\bm\xi}\circ\widehat{{\bm\beta}}_{0}$, where 
the non-negative parameter ${\bm\xi}$ is called the garrote parameter
and the operator $\circ$ denotes the element-wise product of two vectors, i.e., the Hadamard product. 
The optimal garrote parameter $\widehat{\bm{\xi}}$ is found by minimizing the empirical squared loss
with the non-negative constraint $\bm{\xi}\geq\0$ and the $\ell_1$-regularization $\|\bm{\xi}\|_1\leq c$. 
Eventually, the sparse vector $\widehat{\bm{\xi}}\circ\widehat{\bm{\beta}}_0$ is obtained
as the estimator of the coefficient vector. 

We incorporate the NNG into kernel methods. 
In order to induce the sparse estimator, the NNG seems to be more adequate than lasso-type estimator, since
lasso estimator is basically available to the feature selection of linear regression 
models~\cite{JRSS:Tibshirani:1996}. 
The linear model with the garrote parameter 
is expressed as $(\bm{\xi}\circ\x)^T\widehat{\bm{\beta}}_0$. 
Likewise, given the kernel-based estimator $f(\z)$ for $f\in\mathcal{H}$, 
the garrote parameter $\bm{\xi}$ is introduced as the form of $f_{\bm{\xi}}(\z):=f(\bm{\xi}\circ\z)$. 
Both $f\in\mathcal{H}$ and $\bm{\xi}$ can be found by minimizing the empirical loss  
in which $f$ is replaced with $f_{\bm{\xi}}$. 
Here, we propose a simplified two-stage kernel-based estimator with NNG. 
The detail is presented in Algorithm~\ref{alg:DRNNG}.  
In the algorithm, $\eta$ is a positive regularization parameter that controls the sparsity of the variable
selection. Using the representer theorem, the estimator is typically expressed as 
$\widehat{f}_{\widehat{\bxi}}(\z)=\sum_i \alpha_{i}k(\z_i,\widehat{\bxi}\circ\z)$,
where $\z_i$ is a data point and $\alpha_i$ is the estimated parameter in Step 1. 

\begin{algorithm}[tb]
   \caption{Two-stage kernel-based estimator with NNG. } 
   \label{alg:DRNNG}
\begin{algorithmic}
 \STATE {\bfseries Input:} Training samples, and regularization parameters, $\lambda$ and $\eta$. 
 \STATE {\bfseries Step 1:} Find the kernel-based estimator $\widehat{f}$ by solving
 \eqref{eqn:general-kernel-based-estimator}. 
 \STATE {\bfseries Step 2:} Let us define $\widehat{f}_{\bm{\xi}}(\z)=\widehat{f}(\bm{\xi}\circ\z)$. 
 Find the optimal garrote parameter $\widehat{\bm{\xi}}$ by solving 
 \begin{align*}
  &\min_{\bm{\xi}}
  \widehat{L}(\widehat{f}_{\bm{\xi}})+\eta\|\bm{\xi}\|_1,\quad \mathrm{s.t.}\ \ \bm{\xi}\in[0,1]^d. 
 \end{align*}
 \vspace*{-4mm}
 \STATE {\bfseries Output:} The estimator $\widehat{f}_{\widehat{\bm{\xi}}}(\z)$. 
\end{algorithmic}
\end{algorithm}

In the learning algorithm, $\bm{\xi}$ is optimized under the box constraint
$\bm{\xi}=(\xi_1,\ldots,\xi_d)\in[0,1]^d$. 
Here, we introduce the upper constraint $\bm{\xi}\leq\1$ that does not appear in the original NNG.
This is because we need the contraction condition in Section~\ref{subsec:PowerSeriesKernels}
to ensure that the domain of $f_{\bxi}$ is properly defined from that of $f$. 
The estimated function $\widehat{f}_{\widehat{\bxi}}$ depends only on the variables having positive garrote parameter. 
In step~2 of Algorithm~\ref{alg:DRNNG}, one can use the standard optimization methods such as the limited-memory BFGS method
with the box-constraint. Usually, the objective function is not convex.
Practical methods including the multi-start technique should be implemented. 


\subsection{Power Series Kernels and its Invariant Property}
\label{subsec:PowerSeriesKernels}
The statistical model of the learning algorithm is expressed as
$\widetilde{H}=\cup_{\bm{\xi}\in[0,1]^d}\mathcal{H}_{\bm{\xi}}$,  
where $\mathcal{H}_{\bm{\xi}}=\{f_{\bm{\xi}}(\z)\,:\,f\in\mathcal{H}\}$, 
i.e., the multiple kernel model is employed in our method.
In what follows, we assume that the domain of the data, $\mathcal{Z}$, is a compact set
included in $(-1,1)^d$ and satisfies the \emph{contraction condition}, 
${\bm\xi}\circ\mathcal{Z}:=\{{\bm\xi}\circ\z:\z\in\mathcal{Z}\}\subset\mathcal{Z}$ for any $\bm{\xi}\in[0,1]^d$. 
Due to this condition, $f_{\bxi}$ is properly defined without expanding the domain of $f$. 

When the invariant property, $\mathcal{H}_{\bm{\xi}}\subset\mathcal{H}$, holds for all $\bxi\in[0,1]^d$,
we have $\widetilde{\mathcal{H}}=\mathcal{H}$. 
As a result, we can circumvent the computation of multiple kernels. 
For example, the RKHS endowed with the polynomial kernel $k(\x,\z)=(1+\x^T\z)^\ell,\,\ell\in\mathbb{N}$,
agrees with this condition. 
Also, the exponential kernel $k(\x,\z)=\exp(\gamma\x^T\z), \gamma>0$ which is a universal kernel
\cite{steinwart2008support} has the same property. 
On the other hand, the invariant property does not hold for the Gaussian kernel, since the constant function obtained by setting $\bm{\xi}=\0$ is 
not included in the corresponding RKHS~\cite[Corollary~4.44]{steinwart2008support}. 
As the result, we find $\mathcal{H}\neq \widetilde{\mathcal{H}}$ for the Gaussian kernel. 

In general, the class of power series kernels (PS-kernels) \cite{zwicknagl09:_power_series_kernel} is 
available in our method. 
The power series kernel $k(\x,\y)$ for $\x,\y\in\mathcal{Z}\subset(-1,1)^d$ is defined as the power series,
i.e., $k(\x,\y)=\sum_{\balpha\in\mathbb{N}_0^d} w_\balpha{\x^\balpha \y^\balpha}/(\balpha!)^2$, 
where $\x^\balpha=x_1^{\alpha_1}\cdots x_d^{\alpha_d},\, \balpha!=\alpha_1!\cdots\alpha_d!$ and 
$\mathbb{N}_0$ is the set of all non-negative integers. The multi-index sequence $w_{\balpha}$ consists of 
non-negative numbers such that $\sum_{\balpha\in\mathbb{N}_0^d}w_{\balpha}/(\balpha!)^2<\infty$. 
The polynomial and exponential kernels are included in the class of the PS-kernels. 
The native space is defined as $\mathcal{H}_k=\{f(\x)=\sum_{\balpha}c_\balpha \x^\balpha \,\big|\, 
\sum_{\balpha}(\balpha!)^2 c_\balpha^2/w_{\balpha}<\infty\}$,
where $\balpha\in\mathbb{N}_0^d$ in the summation runs on multi-indices with  $w_{\balpha}>0$. 
As shown in \cite{zwicknagl09:_power_series_kernel},  $\mathcal{H}_k$ is the RKHS with the inner product
$\<f,g\>=\sum_{\balpha}(\balpha!)^2{c_{\balpha} d_{\balpha}}/w_{\balpha}$ for
$f(\x)=\sum_{\balpha}c_{\balpha}\x^\balpha$ and $g(\x)=\sum_{\balpha}d_{\balpha}\x^\balpha$.
The invariant property holds for the RKHS endowed with the power series kernel. 
Indeed, for $f(\x)=\sum_{\balpha\in\mathbb{N}_0^d}c_\balpha \x^\balpha\in\mathcal{H}_k$, 
the coefficients of $f_{\bxi}$ are given as $c_\balpha\cdot\bxi^\balpha,\, \balpha\in\mathbb{N}_0^d$. 
Since $\bxi\in[0,1]^d$, we have $|c_\balpha\cdot\bxi^\balpha|\leq |c_\balpha|$.
Thus, $f_{\bxi}\in\mathcal{H}_k$ holds. Moreover we have $\|f_{\bxi}\|\leq \|f\|$. 
Some other properties of the power series kernels are presented in 
Appendix~\ref{Properties_PSKernels}. 


\section{Theoretical Results : Variable Selection Consistency}
\label{theory}
Let us consider the variable selection consistency of Algorithm~\ref{alg:DRNNG}.
All proofs are found in 
Appendix~\ref{appendix:proofs}. 
Suppose that the target function $f^*(\z)$ for $\z=(z_1,\ldots,z_d)$ essentially depends on the variables $z_1,\ldots,z_s$, where $s\leq d$. 
The \emph{variable selection consistency} means that for the estimated garrote parameter
$\widehat{\bxi}=(\widehat{\xi}_1,\ldots,\widehat{\xi}_d)$, 
the probability of the event $\{j:\widehat{\xi}_j>0\}=\{1,\ldots,s\}$ tends to one as the sample size $n$ goes to infinity.

Let us define some notations and introduce some assumptions. 
\begin{assmp}
 \label{assump:uniform-consistency}
 The kernel-based estimator $\widehat{f}$ using \eqref{eqn:general-kernel-based-estimator} 
 has the property of the statistical consistency for the target $f^*\in\mathcal{H}$
 in the RKHS norm, i.e., $\|\widehat{f}-f^*\|=o_P(1)$. 
\end{assmp}
In the following we assume that the kernel function and its derivatives are bounded by the constant $\kappa>0$, i.e.,
$\sup_{\z\in\mathcal{Z}}\sqrt{k(\z,\z)} \leq \kappa$ and 
$\sup_{\z\in\mathcal{Z}}\sqrt{\partial_{a}\partial_{d+a}k(\z,\z)}\leq \kappa$ hold for $a=1,\ldots,d$. 
We use the similar inequality up to required order of derivatives. Then, 
the convergence of the RKHS norm $\|\widehat{f}-f^*\|$ leads to that of $\|\widehat{f}-f^*\|_{\infty}$ and 
$\|\partial_a\widehat{f}-\partial_{a}f^*\|_{\infty}$. 
Indeed, the generalized reproducing property \cite{zhou08:_deriv} leads to
$\|\widehat{f}-f^*\|_{\infty} \leq \sup_{\z\in\mathcal{Z}}\sqrt{k(\z,\z)}\|\widehat{f}-f^*\|$ and
$\|\partial_a\widehat{f}-\partial_{a}f^*\|_{\infty} \leq
\sup_{\z\in\mathcal{Z}}\sqrt{\partial_{a}\partial_{d+a}k(\z,\z)}\|\widehat{f}-f^*\|$. 
The same inequality holds for $\widehat{f}_{\bm\xi}$ and ${f}_{\bm\xi}^*$, since
$\|\widehat{f}_{\bm\xi}-{f}_{\bm\xi}^*\|\leq \|\widehat{f}-{f}^*\|$ holds for the PS-kernels. 

Let us define $L(\bm{\xi};f)$ (resp. $\widehat{L}(\bm{\xi};f)$) for $\bm{\xi}\in[0,1]^d$ and 
$f\in\mathcal{H}$ as $L(f_{\bm{\xi}})$ (resp. $\widehat{L}(f_{\bm{\xi}})$).
When $f(\z)$ does not depend on $z_\ell$, the derivative $\frac{\partial}{\partial\xi_\ell}L(\bm{\xi};f)$ vanishes. 
The assumption $f^*\in\mathcal{H}$ leads that the optimal solution of the problem $\min_{f\in\mathcal{H}}L(\1;f)$ with
$\1=(1,\ldots,1)\in[0,1]^d$ is $f^*$. Since the target function $f^*(\z)$ depends only on $z_1,\ldots,z_s$, 
the parameter $\bm{\xi}^*=(\bm{\xi}_1^*, \bm{\xi}_0^*)=(\1,\0)\in\Rbb^{s}\times\Rbb^{d-s}$ is 
an optimal solution of the problem $\min_{\bm{\xi}\in[0,1]^d}L(\bm{\xi};f^*)$. Let us make the following
assumptions on the above loss functions. 
\begin{assmp}
 \label{assump:lossfunc}
 (a) For any $\varepsilon>0$, $L(\bm{\xi};f^*)$ satisfies 
 $L((\bm{\xi}_1^*,\0);f^*) <
 \inf\{L((\bm{\xi}_1,\0);f^*)\,:\,\bm{\xi}_1\in[0,1]^s, \|\bm{\xi}_1-\bm{\xi}_1^*\|\geq \varepsilon\}$. 
 (b) The uniform convergence of the empirical loss
 $\sup_{\bm{\xi}\in[0,1]^d}|\widehat{L}(\bm{\xi};\widehat{f})-\widehat{L}(\bm{\xi};f^*)|=o_P(1)$
 holds for the kernel-based estimator $\widehat{f}$ and the target function $f^*$.
 Also, the uniform convergence at $f=f^*$ holds, i.e., $\sup_{\bm{\xi}\in[0,1]^d}|\widehat{L}_n(\bm{\xi};f^*)-L(\bm{\xi};f^*)|=o_P(1)$. 
\end{assmp}
\begin{assmp}
 \label{assump:loss_deriv}
For the kernel-based estimator $\widehat{f}$ and the target function $f^*$, 
(a) the derivative of the empirical loss satisfies 
 $\sup_{\bm{\xi}\in[0,1]^d}|\frac{\partial}{\partial\xi_i} \widehat{L}(\bm{\xi};\widehat{f})-\frac{\partial}{\partial\xi_i} \widehat{L}(\bm{\xi};f^*)|=O_P(\delta_n)$, where $\delta_n\searrow0$ as the sample size $n$ tends to infinity. 
Also, (b) the uniform convergence at $f=f^*$, i.e., 
 $\sup_{\bm{\xi}\in[0,1]^d}|\frac{\partial}{\partial{\xi}_i}\widehat{L}(\bm{\xi};f^*)-\frac{\partial}{\partial{\xi}_i}{L}(\bm{\xi};f^*)|=O_P(\delta_n')$ 
holds, where $\delta_n'\searrow0$ as $n\rightarrow\infty$. 
\end{assmp}
Under the above assumptions, we prove the variable selection consistency of $\widehat{\bm{\xi}}$. 
\begin{thm}
 \label{thm:var_consistency}
 (i) Under Assumptions~\ref{assump:uniform-consistency} and \ref{assump:lossfunc}, 
 $\widehat{\bm{\xi}}_1$ in $\widehat{\bm{\xi}}=(\widehat{\bm{\xi}}_1,\widehat{\bm{\xi}}_0)$ 
 has the statistical consistency, i.e., $\widehat{\bm{\xi}}_1$ converges to $\bm{\xi}_1^*=\1\in\Rbb^s$ in probability. 
 (ii) Assume the Assumptions~\ref{assump:uniform-consistency}, \ref{assump:lossfunc}, and \ref{assump:loss_deriv}. 
 Suppose that $\lim_{n\rightarrow\infty}(\delta_n+\delta_n')/\eta_n=0$, where $\delta_n$ and $\delta_n'$ are
 positive sequences in Assumption~\ref{assump:loss_deriv}. 
 Then, $\widehat{\bm{\xi}}_0=\bm{\xi}_0^*=\0\in\Rbb^{d-s}$ holds with high probability when the sample size is sufficiently large. 
\end{thm}
The proof is in 
Appendix~\ref{appendix:variable-consistency} 
follows the standard argument of the statistical consistency of
M-estimators shown in~\cite[Theorem 5.7]{vaart00:asympstatis}. 
The order of $\eta_n$ should be greater than $\delta_n$ and $\delta_n'$ in order to draw $\widehat{\bm{\xi}}_0$ to $\0$. 
All kernel-based estimators in Section~\ref{Problem_Setup} satisfies Assumptions~\ref{assump:uniform-consistency}.
In 
Appendix~\ref{appendix:cond-a-proof}, 
we prove that for each estimator in Section~\ref{Problem_Setup}, the condition (a) in
Assumptions~\ref{assump:lossfunc} holds under a mild assumption. 

We show sufficient conditions of Assumptions~\ref{assump:lossfunc} (b) and \ref{assump:loss_deriv}, 
when the quadratic loss functions in \eqref{eqn:quadratic-loss} and \eqref{eqn:expected-quadratic-loss} are used. 
Below, $\|C\|$ denotes the operator norm defined from the norm on $\mathcal{H}$. 
\begin{lem}
 \label{prop:conv-operator_Assump2}
 Let $k$ be the PS kernel. 
 We assume that $\|C\|<\infty$ and that $\|\widehat{C}-C\|$ and $\|\widehat{g}-g\|$ converge to zero in
 probability as sample size tends to infinity. 
 Then, under Assumption~\ref{assump:uniform-consistency}, the uniform convergence condition (b)
 in Assumption~\ref{assump:lossfunc} holds.
\end{lem}
\begin{lem}
 \label{prop:conv-operator_Assump3}
 Let $k$ be the differentiable PS kernel. 
 Suppose that $\z_1,\ldots,\z_n$ are i.i.d. samples from $p(\z)$ and that 
 $\z_1',\ldots,\z_{n'}'$ are i.i.d. samples from $q(\z')$.
 Let $\mathcal{I}$ and $\mathcal{J}$ are collections of finite subsets in $\{1,\ldots,d\}$. 
 Let us define
 $\widehat{C}=\frac{1}{n}\sum_{\ell=1}^{n}\sum_{I\in\mathcal{I}}h_I(\z_\ell)\partial_{I}k(\z_\ell,\cdot)\otimes\partial_{I}k(\z_\ell,\cdot)$
 and
 $\widehat{g}=\frac{1}{n'}\sum_{\ell=1}^{n'}\sum_{J\in\mathcal{J}}\bar{h}_J(\z_\ell')\partial_{J}k(\z_\ell',\cdot)$,
 where $h_I, I\in\mathcal{I}$ and $\bar{h}_J, J\in\mathcal{J}$ are bounded functions on $\mathcal{Z}$. 
 The operator $C$ and the element $g\in\mathcal{H}$ are defined by the expectation of $\widehat{C}$ and $\widehat{g}$, respectively. 
 Then, Assumption~\ref{assump:loss_deriv} holds under Assumption~\ref{assump:uniform-consistency}. 
\end{lem}
Loss functions for the estimators in Section~\ref{Problem_Setup} are expressed using
the above $\widehat{C}, C, \widehat{g}$ and $g$. 
The stochastic convergence property of $\|\widehat{C}-C\|$ and $\|\widehat{g}-g\|$
is guaranteed from Theorem~7 of \cite{rosasco10:_learn_integ_operat}. 
Hence, the sufficient condition in Proposition~\ref{prop:conv-operator_Assump2} is satisfied. 

Let us show another sufficient condition of Assumption~\ref{assump:loss_deriv}. 
We deal with more general operators $C$ and $\widehat{C}$, while we need an additional smoothness assumption on the target function $f^*$.
\begin{lem}
 \label{prop:conv-operator_Assump3_add}
 Let $k$ be the PS kernel of the RKHS $\mathcal{H}$. 
 For the linear operators $C$ and $\widehat{C}$, and the elements $g$ and $\widehat{g}$ in $\mathcal{H}$, 
 suppose that the inequalities, 
 $|\big\<f,\widehat{C}h\big\>|\leq \beta\|f\|_\infty\|h\|_\infty,\ 
 |\big\<f,C h\big\>|\leq \beta\|f\|_\infty\|h\|_\infty,\ 
 |\big\<f,\widehat{g}\>|\leq \beta\|f\|_\infty$, and 
 $|\big\<f,g\>|\leq \beta\|f\|_\infty$ 
 hold for any $f,h\in\mathcal{H}$, where $\beta$ is a positive constant. 
 We assume that for any training data,
 the derivatives of quadratic loss functions, 
  $\frac{\partial}{\partial\xi_i}\widehat{L}(\bxi;f)$ for $f\in\mathcal{H}_k$ and
 $\frac{\partial}{\partial\xi_i}L(\bxi;f^*)$, 
 are continuous as the function of $\bxi$ over the closed hypercube $[0,1]^d$.
 Then, (i) under Assumption~\ref{assump:uniform-consistency}, 
 the condition (a) in Assumption~\ref{assump:loss_deriv} holds.
 (ii) Suppose that $\|\widehat{C}-C\|$ and $\|\widehat{g}-g\|$ converge to zero in 
 probability as the sample size tends to infinity. 
 When $\sup_{\bxi\in(0,1)^d}\|\frac{\partial f_{\bxi}^*}{\partial\xi_i}\|<\infty$ holds for all
 $i=1,\ldots,d$,
 the condition (b) in Assumption~\ref{assump:loss_deriv} holds.
\end{lem}
The operators $C, \widehat{C}$ and the elements $g$ and $\widehat{g}$ in Proposition~\ref{prop:conv-operator_Assump3} satisfy
the inequalities of the assumption in Lemma~\ref{prop:conv-operator_Assump3_add}. 


\section{Experimental Results}
\label{experiment}
In this section we show the empirical performance of the kernel-based density-ratio estimator. 
In section~\ref{exp:syn}, we conduct the synthetic data analysis to mainly evaluate the variable selection consistency. 
In section~\ref{exp:real}, we analyze three real datasets, PCAS dataset~\cite{nishikimi2017novel},  
diabetes dataset~\cite{strack2014impact} and Wisconsin breast cancer dataset. 
The latter two datasets are published at the UCI machine learning repository. 
We see that our method can select a practically interpretable subset of features.

\subsection{Synthetic Data Analysis}
\label{exp:syn}

\begin{table*}[t]
 \label{tbl:test_error}
   \centering
 \caption{Test loss, FPR, and FNR for each learning method.
 The dimension of data is $d$, and density ratio depends on $s=5$ variables among $d$ variables. 
 The sample size is set to $n=1000$ and $m=800$. 
 The parameter $c$ corresponds to the discrepancy between two probability densities, $p$ and $q$. 
 Top panel: the results for $d=20, s=5$. Bottom panel: the results for $d=100, s=5$. 
 The bold face in the table shows the learning method such that the sum of FPR and FNR is minimum. 
 }
 {\scriptsize
     \begin{tabular}{l|rrr|rrr|rrr}
     \multicolumn{1}{c}{}& \multicolumn{9}{c}{$d=20,\ s=5$} \\ 
     \cline{2-10}
	\multicolumn{1}{c}{}   & \multicolumn{3}{c}{$c=0.1$}
   &  \multicolumn{3}{c}{$c=0.3$}
   &  \multicolumn{3}{c}{$c=0.5$}     \\ \hline
           estimator:$\eta$ & test loss      &  FPR & FNR & test loss      &  FPR & FNR& test loss      &  FPR & FNR \\ \hline
   SLR     & $-0.415$& \bm{$0.442$}& \bm{$0.000$} &  $ 0.511$ &$0.449$& $0.00$ & $ 9.198$& $0.482$ &  $0.000$ \\   
exp:0.01   & $ 3.212$& $1.000$& $0.000$ &  $ 1.634$ &$1.000$& $0.000$ & $ 0.276$& $1.000$ &  $0.000$ \\
exp:0.1    & $-0.493$& $0.000$& $0.920$ &  $ 1.583$ &$0.020$& $0.000$ & $ 0.278$& $0.936$ &  $0.000$ \\
exp:0.5    & $-0.500$& $0.000$& $1.000$ &  $ 0.296$ &\bm{$0.000$}& \bm{$0.000$} & $ 0.273$& \bm{$0.000$} & \bm{$0.000$} \\
exp:1      & $-0.500$& $0.000$& $1.000$ &  $-0.326$ &$0.000$& $0.260$ & $ 0.264$& \bm{$0.000$} &  \bm{$0.000$} \\
gauss:0.01 & $-0.252$& $0.964$& $0.180$ &  $-0.390$ &$1.000$& $0.000$ & $-0.448$& $1.000$ &  $0.000$ \\
gauss:0.1  & $-0.391$& $1.000$& $0.000$ &  $-0.392$ &$1.000$& $0.000$ & $-0.450$& $1.000$ &  $0.000$ \\
gauss:0.5  & $-0.383$& $0.984$& $0.027$ &  $-0.420$ &$1.000$& $0.000$ & $-0.448$& $1.000$ &  $0.000$ \\
gauss:1    & $-0.393$& $0.347$& $0.733$ &  $-0.427$ &$1.000$& $0.000$ & $-0.451$& $1.000$ &  $0.000$ \\ \hline \hline
     \end{tabular} 
 \vspace*{2mm}
\begin{tabular}{l|rrr|rrr|rrr}
   \multicolumn{1}{c}{}& \multicolumn{9}{c}{$d=100,\ s=5$} \\ 
   \cline{2-10}
   \multicolumn{1}{c}{}   & \multicolumn{3}{c}{$c=0.1$}
       &  \multicolumn{3}{c}{$c=0.3$}
   &  \multicolumn{3}{c}{$c=0.5$}     \\ \hline
           estimator:$\eta$ & test loss      &  FPR & FNR & test loss      &  FPR & FNR& test loss      &  FPR & FNR \\ \hline
   SLR      & $-0.446$& $0.132$& $0.013$ & $ 0.390$& $0.177$& $0.000$&  $  3.095$& $0.198$& $0.000$ \\
exp:0.01    & $-0.453$& \bm{$0.076$}& \bm{$0.007$} & $-0.351$& $1.000$& $0.000$&  $ -0.424$& $1.000$& $0.000$ \\
exp:0.1     & $-0.500$& $0.000$& $0.933$ & $-0.403$& \bm{$0.000$}& \bm{$0.000$}&  $ -0.432$& $0.001$& $0.000$ \\
exp:0.5     & $-0.500$& $0.000$& $1.000$ & $-0.481$& $0.000$& $0.053$&  $ -0.437$& \bm{$0.000$}& \bm{$0.000$} \\
exp:1       & $-0.500$& $0.000$& $1.000$ & $-0.500$& $0.000$& $0.893$&  $ -0.452$& \bm{$0.000$}& \bm{$0.000$} \\
gauss:0.01  & $-0.483$& $1.000$& $0.000$ & $-0.493$& $1.000$& $0.000$&  $ -0.497$& $1.000$& $0.000$ \\
gauss:0.1   & $-0.485$& $1.000$& $0.000$ & $-0.493$& $1.000$& $0.000$&  $ -0.498$& $1.000$& $0.000$ \\
gauss:0.5   & $-0.491$& $0.224$& $0.880$ & $-0.495$& $0.885$& $0.380$&  $ -0.498$& $1.000$& $0.000$ \\
gauss:1     & $-0.500$& $0.000$& $1.000$ & $-0.497$& $0.314$& $1.000$&  $ -0.497$& $0.851$& $0.333$ \\ \hline
\end{tabular}
 }
\end{table*}

We report numerical experiments using synthetic data. 
The purpose of this simulation is to confirm the statistical accuracy of the KuLSIF with NNG 
in terms of the variable selection in density ratio. 

Suppose that $\x_1,\ldots,\x_n$ were generated from the $d$-dimensional normal distribution 
with mean $\bm{\mu}$ and the variance-covariance matrix $I_d$, i.e., $N_d(\bm{\mu},I_d)$.
This distribution corresponds to $p$ in the denominator of the density ratio. 
Likewise, suppose that $\y_1,\ldots,\y_m$ were generated from the probability $q$
that is defined as the $d$-dimensional standard normal distribution $N_d(\0,I_d)$.
Here, $n=1000$ and $m=800$. 
Since $w^*(\z)=q(\z)/p(\z)$ is proportional to $\exp(-\z^T\bm{\mu})$, 
$w^*(\z)$ depends only on the variables such that the $i$-th component $\mu_i$ does not vanish. 
The $d$-dimensional vector $\bm{\mu}$ was set to $(c\,\1,\,\0)$, where $\1=(1,\ldots,1)\in\Rbb^s$ and $c$ was set to $0.1, 0.3$ or $0.5$. 
When the training data was fed to the learning algorithm, each component of data was scaled to zero mean and
unit variance. This scaling was introduced to stabilize the calculation of the exponential function. 

Our methods were compared with the sparse logistic regression (SLR) estimator, that is given as the logistic 
regression with the $\ell_1$-regularization. 
Originally, the SLR estimator is the learning algorithm for sparse classification problems. 
Suppose that the samples from $p(\x)$ (resp. $q(\x)$) have the label $+1$ (resp. $-1$). 
Then, the ratio $\Pr(-1|\x)/\Pr(+1|\x)$ of the estimated logistic model
$\Pr(+1|\x)=1/(1+\exp(-(\beta_0+\bm{\beta}^T\x)))$ is approximately proportional to the density ratio $q/p$
if the statistical model is specified. The $\ell_1$-regularization $\|{\bm\beta}\|_1$ to the weight vector
induces the sparse solution as well as the lasso estimator.
The regularization parameter for the $\ell_1$ regularization in SLR was determined by the cross validation. 
In our method, the hyper-parameter $\lambda$ of KuLSIF in step~1 
was set to $1/(n\wedge{m})^{0.9}$, which guarantees the statistical consistency of KuLSIF under a proper
assumption~\cite{kanamori12:_statis}. 
In step~2 of KuLSIF with NNG, the regularization parameters $\eta$ varies from 0.01 to 1.
For the computation, we used the R language. The {\tt{glmnet}} package of the R language was used for the SRL.
For the optimization in step 2 of KuLSIF with NNG,
we used the limited-memory BFGS method with the box-constraint implemented in {\tt{optim}} function in R. 

We report the result of the numerical experiments. 
The test loss of the estimated density ratio, $\widehat{w}$, was measured by the shifted squared loss,
$L(\widehat{w}) = \frac{1}{2}\int{}p(\x)\widehat{w}(\x)^2d\x-\int{}q(\y)\widehat{w}(\y)d\y$.
The variable selection accuracy was evaluated by the false positive rate (FPR) and false negative rate (FNR). 
In each problem setup, the simulation was repeated 30 times. 
The Table~\ref{tbl:test_error} shows the test loss, the FPR, and FNR averaged over the repetitions. 

The result indicates that the KuLSIF with NNG using the exponential kernel performs better than the other
method in terms of the variable selection if the regularization parameter $\eta$ is properly chosen. 
When we use the Gaussian kernel with our method, we cannot expect the ability of the variable selection.
The variables selected by the SLR had relatively high FPR. In the experiment, the regularization parameter 
of SLR was chosen by the cross validation under the empirical 0/1 loss. In this case, 
a bigger subset of variables tends to be chosen. 
The test loss of the SLR is larger than the other methods, since the classification-oriented learning
algorithms are considered not to be necessarily suit to the density ratio estimation. 


\subsection{Real Application : Covariate-shift Adaptation for Binary Classification}
\label{exp:real}

In this section, we show the results of applying the proposed method to {\it learning under the covariate shift 
problem}~\cite{shimodaira2000improving, sugiyama2012density}. 
The covariate shift is a phenomenon in which the distribution of covariates changes between learning phase and 
test phase, and it has been proved that the covariate shift can ``adapt'' by weighting the sample with density-ratio.  
We consider the binary classification problem under the covariate-shift, and compare the test accuracy of the five 
scenaros:(i) no adaptation, (ii)  adapting with the density-ratio using all variables, (iii) adapting with the 
density-ratio using selected variables by kernel NNG (proposed method), (iv) adapting with the density-ratio using selected variables by Lasso and  
 (v) adapting with the density-ratio using selected variables by sequential forward selection (SFS). 
Here we employ the logistic regression to learn a classification model.   
In this study, we analyzed four medical datasets: post-cardiac arrest syndrome (pcas)~\cite{nishikimi2017novel}, 
chronic kidney disease (ckd), cervical cancer (ccancer) and  cortex nuclear (cnuclear)~\cite{Dua:2017}. The characteristics of each dataset are summarized in Table~\ref{tb:character}. 
In the column of the sample size ($\#$samples) and the number of features ($\#$fearures), 
those after preprocessing, such as missing value correction, are entered.
In each dataset, we stratified the data into two groups by a certain feature, and learning was performed on 
one group, and the classification accuracy was evaluated on the other. 
The stratification factor is summarized in the last column of Table~\ref{tb:character}. 
The learning phase consists of kernel 
density-ratio estimation (KuLSIF) step and weighted empirical risk minimization (weighted logistic regression) step. 
All experiments are implemented in Python 3.6.1. We use scikit-learn~\cite{scikit-learn} for Lasso 
and mlxtend~\cite{raschkas_2016_594432} for SFS. 
Similar to the previous section, we used the limited-memory BFGS method with the box-constraint in {\tt{minimize}} function 
in Scipy. 

\begin{table}[t]
        \begin{center}
     \tblcaption{\label{tb:character} Summary of the datasets}
      \begin{tabular}{c|c|c|c}
 	\hline
     & $\#$samples & $\#$features & stratification \\ \hline
    pcas  & 151 & 17 & mydriasis ($0/1$)  \\ \hline 
    ckd  & 251 & 17 & pc ($0/1$)  \\ \hline 
    ccancer  & 668  & 27 & age ($\gtrless30$)  \\ \hline
    cnuclear & 1047 & 71 & class ($0\sim7$) \\ \hline
  \end{tabular}
  \end{center}
\end{table}

\begin{figure}[H]
 \vspace{-3pt}
 \begin{center}
 \begin{tabular}{cc}
  \includegraphics[bb=0 0 700 520, clip, scale=0.5]{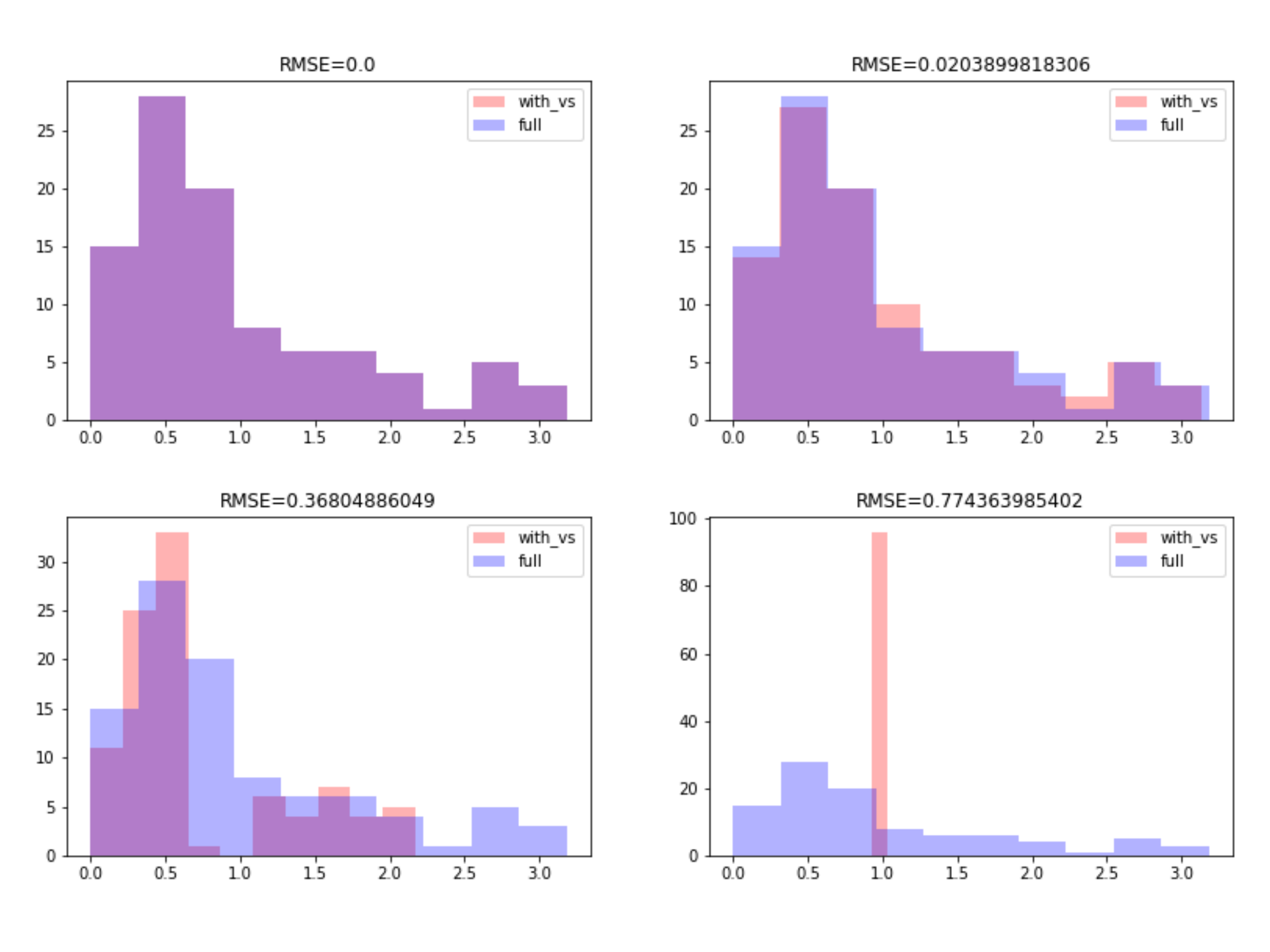}
 \end{tabular}
  \caption{The results of the proposed method with varying regularization parameter $\eta$. Top left : $\eta = 0.0001$, 
  top right : $\eta = 0.01$ (best), bottom left : $\eta = 0.1$, bottom right : $\eta = 1.0$. 
  The corresponding RMSE are $0.0$, $0.02$, $0.36$ and $0.77$ respectively. The x-axis represents the value of the density ratio, and the y-axis represents 
the frequency of the value of the density ratio in a certain interval.
  \label{density_ratio_parameters} }
 \end{center}
\end{figure}

The results of the experiments are summarized in Figure~\ref{density_ratio_parameters}, \ref{density_ratio_comparison} and Table~\ref{tb:real}. 
Figure~\ref{density_ratio_parameters} shows the differences of the density ratio estimated by selected variables (red, proposed method) and 
by full variables (blue) for the pcas dataset. 
The differences of each histogram are measured by root mean square deviation (RMSE). 
From top left to bottom right, we vary the regularization parameter $\eta$ from $0.0001$ to $1.0$. 
When we use small $\eta$, all variables are left and the corresponding value of density ratio are equal to that of using full variables (top left). 
On the other hand, large $\eta$ excludes of all variables and leads the constant density ratio (bottom right). 
Here we selected an regularization parameter $\eta$ appropriately by grid search in each dataset. 

\begin{figure}[t]
 \vspace{-3pt}
 \begin{center}
  \begin{tabular}{cc}
 \hspace{-0.8cm}
  \includegraphics[bb=0 0 700 710, clip, scale=0.35]{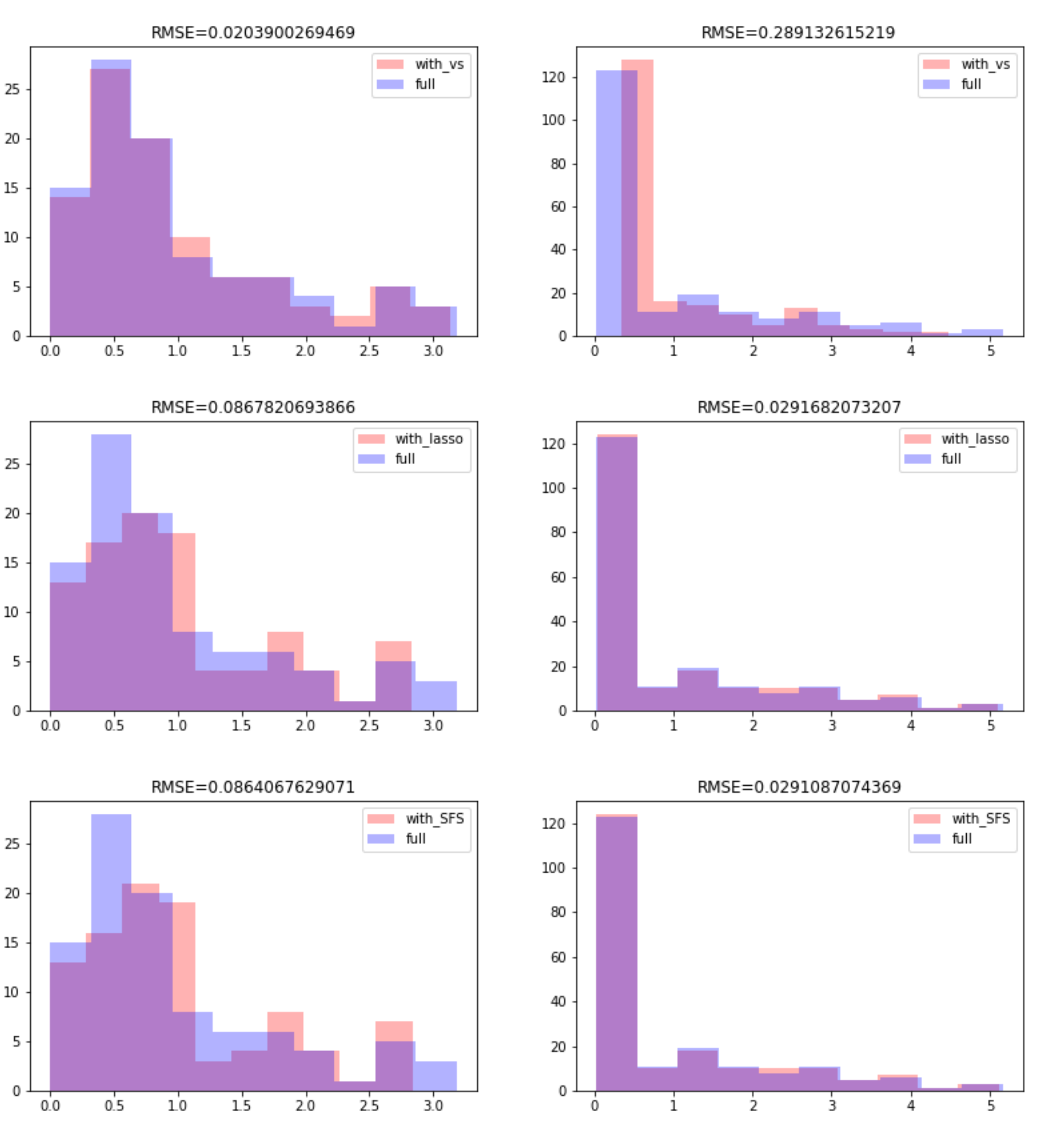}&
  \hspace{-0.8cm}
    \includegraphics[bb=0 0 700 710, clip, scale=0.35]{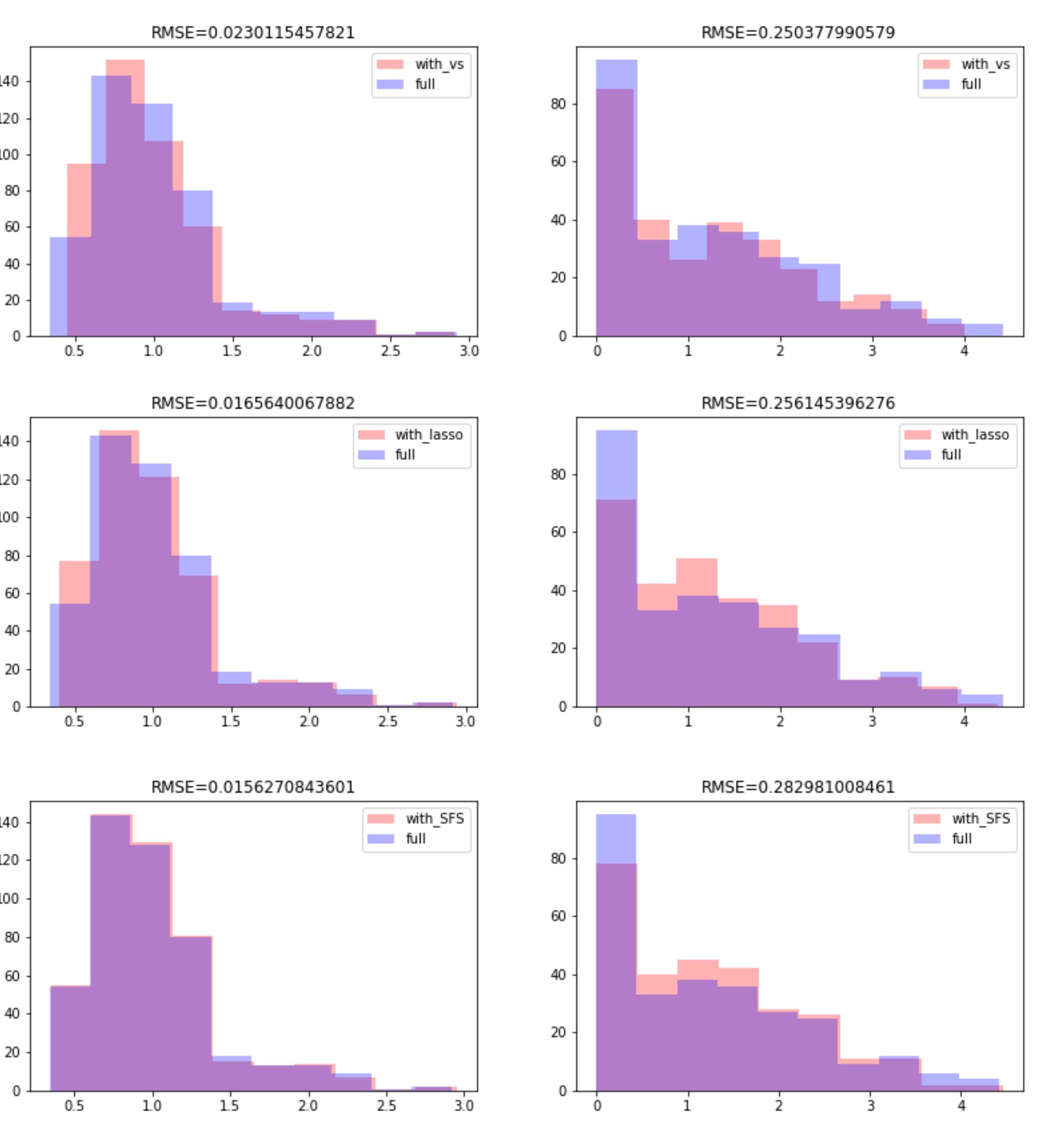}
 \end{tabular}
  \caption{Comparison of the estimated density-ratio between using all variables and using selected variables. 
  Top row : kernel NNG (proposed method), mid row : Lasso, bottom row : SFS. 
  From left to right, each column shows the result of pcas, ckd, ccancer and cortex nuclear dataset respectively. The x-axis represents the value of the density ratio, and the y-axis represents the frequency of the value of the density ratio in a certain interval.
  \label{density_ratio_comparison} }
 \end{center}
\end{figure}

\begin{table}[t]
        \begin{center}
         {
      \tblcaption{\label{tb:real}Comparison of classification accuracy}
      \begin{tabular}{c|c|c|c|c|c}
      \hline
&  kernel NNG  & Lasso & SFS & adapt full & no adapt  \\ \hline
    pcas & {\bf 0.830} (13/17)  & {\bf 0.830} (17/17) & {\bf 0.830} ({\bf 12/17}) & 0.830 & 0.773  \\ 
    ckd & {\bf 0.981} ({\bf 8/17}) & {\bf 0.981} (17/17) & {\bf 0.981} (16/17) & 0.981 & 0.943  \\ 
    ccancer & {\bf 0.884} ({\bf 4/27}) & 0.879 (19/27) & 0.879 (22/27) & 0.879 & 0.884 \\ 
   cnuclear & {\bf 0.634} (60/71) & 0.599 (62/71) & 0.603 ({\bf 37/71}) & 0.619 & 0.587 \\ \hline
      \end{tabular}
      }
  \end{center}
\end{table}


Table~\ref{tb:real} shows the test accuracy in each scenario: 
covariate shift adaptation and variable selection by proposed method (kernel NNG), 
covariate-shift adaptation and variable selection by Lasso (Lasso), 
covariate-shift adaptation and variable selection by SFS (SFS)
covariate-shift adaptation using all features (adapt full) and 
no covariate-shift adaptation (no adapt). 
The numbers in parentheses indicate the number of selected variables / the number of all variables in the density-ration estimation step. 
The bold symbol represents the best performance.
As shown in the results, the proposed method appropriately selected a small number of variables 
without deteriorating the classification accuracy.  

Figure~\ref{density_ratio_comparison} shows the comparison of the estimated density-ratio between 
selected variables (by kernel NNG, Lasso and SFS) and full variables.
The differences of each histogram are also measured by RMSE. 
Although the Lasso and the SFS achieved the small RMSE in ckd and ccancer datasets, 
both methods could not select the features appropriately. 
In pcas dataset, SFS shows the best performance in both the classification accuracy and the number of selected features.  
However, RMSE of the estimated density ratio is worse than the proposed method. 
On the other hand, we can observe that the kernel NNG performs the reasonable variable selection 
without significant change of the density ratio. 



\section{Concluding Remarks}
\label{conclusion}

This paper provided a unified variable selection method for
nonparametric learning by the power series kernel.
The proposed method can be applied to kernel methods using the
quadratic loss such as kernel-ridge regression and kernel
density-ratio estimation.
Theoretically, we proved the variable selection consistency under mild
assumption thanks to the property of the power series kernel.
Experimental results showed that our method efficiently worked for the
variable selection for kernel methods on both synthetic and real-world data. 
In the second step of our method, we need to solve non-convex
optimization problem for variable selection.
In numerical experiments, non-linear optimization algorithms such as
the limited-memory BFGS method showed a good performance.
Future work includes the development of more efficient optimization
methods to deal with learning problems with massive data sets. 

\section*{Acknowledgements}
This research was supported by JST-CREST (JPMJCR1412) from the Ministry of Education,
Culture, Sports, Science and Technology of Japan.


\bibliography{paper}
\bibliographystyle{plain}


\appendix


\section{Statistical consistency of kernel-based density-ratio estimator}
\label{Consistency_KuLSIF}

The proof follows the convergence analysis developed by \cite{Caponnetto:2007:ORR:1290530.1290534,JMLR:v18:16-011}. 
The estimator $\widehat{f}$ is the minimum solution of the function $\widehat{J}(f)$ over $\mathcal{H}$, where 
\begin{align*}
 \widehat{J}(f)= \widehat{L}(f)+\frac{\lambda_n}{2}\|f\|^2. 
\end{align*}
Thus, we have $\widehat{f}=(\widehat{C}+\lambda_n{I})^{-1}\widehat{g}$. 
When the sample size $n=\min\{n_1, n_2\}$ tends to infinity, $\widehat{J}(f)$ converges to $J(f)$ that is defined as 
\begin{align*}
 J(f)
 =
 L(f)+\frac{\lambda_n}{2}\|f\|^2
 =
 \frac{1}{2}\langle{}f,{C}f\rangle-\langle{}g,f\rangle+\frac{\lambda_n}{2}\|f\|^2. 
\end{align*}
The minimizer of $J(f)$ is expressed as $f_{\lambda_n}=(C+\lambda_n{I})^{-1}g$. 
Note that the true density ratio is given by $f^*=f_0=C^{-1}g=q/p$. Consider
\begin{align*}
 \widehat{f}-f_{\lambda_n}
 &=
 (\widehat{C}+\lambda_n{I})^{-1} \big( \widehat{g} - (\widehat{C}+\lambda_n{I}) f_{\lambda_n} \big)  \\
 &=
 (\widehat{C}+\lambda_n{I})^{-1} \big( \widehat{g} - g-(\widehat{C}-C)f_{\lambda_n} \big)  \\
 &=
 (\widehat{C}+\lambda_n{I})^{-1}(\widehat{g} - g)
 -(\widehat{C}+\lambda_n{I})^{-1}(\widehat{C}-C)(f_{\lambda_n}-f_0)
 -(\widehat{C}+\lambda_n{I})^{-1}(\widehat{C}-C)f_0. 
\end{align*}
We define 
\begin{align*}
 S_1&=\| \widehat{C}+\lambda_n{I})^{-1}(\widehat{g} - g)\|,\\
 S_2&=\|(\widehat{C}+\lambda_n{I})^{-1}(\widehat{C}-C)(f_{\lambda_n}-f_0)\|,\\
 S_3&=\|(\widehat{C}+\lambda_n{I})^{-1}(\widehat{C}-C)f_0\|,\\ 
 A(\lambda_n)&=\|f_{\lambda_n}-f_0\|, 
\end{align*}
so that
 \begin{align*}
  \|\widehat{f}_{\lambda_n}-f^* \|  = \|\widehat{f}_{\lambda_n} - f_0 \|\leq S_1+S_2+S_3+A(\lambda_n). 
 \end{align*}
Proposition A.4 in \cite{JMLR:v18:16-011} is used to bound $S_1,S_2$ and $S_3$ as follows, 
\begin{align*}
 S_1&\leq \|(\widehat{C}+\lambda_n{I})^{-1}\| \|\widehat{g} - g\| = O_p(\frac{1}{\lambda_n\sqrt{n}}), \\
 S_2&\leq \|(\widehat{C}+\lambda_n{I})^{-1}\| \|(\widehat{C}-C)(f_{\lambda_n}-f_0)\| = O_p(\frac{A(\lambda_n)}{\lambda_n\sqrt{n}}), \\
 S_3&\leq \|(\widehat{C}+\lambda_n{I})^{-1}\| \|(\widehat{C}-C)f_0\| = O_p(\frac{1}{\lambda_n\sqrt{n}}). 
\end{align*}
Using the above bounds, we obtain 
\begin{align*}
 \|\widehat{f}_{\lambda_n} - f_0\|
 =
 O_p\bigg(\frac{1}{\lambda_n\sqrt{n}}+\frac{A(\lambda_n)}{\lambda_n\sqrt{n}}\bigg)+A(\lambda_n). 
\end{align*}
The asymptotic order of $A(\lambda_n)$ was revealed by Proposition A.3(i) in \cite{JMLR:v18:16-011}. 
Indeed, (I) the proposition shows that $A(\lambda_n)\rightarrow0$ as $\lambda_n\rightarrow0$ 
if $f_0\in\overline{\mathrm{Range}(C)}$; (II) if $f_0\in\overline{\mathrm{Range}(C^\beta)}$ for $\beta>0$, it follows that 
\begin{align*}
 A(\lambda_n)\leq \max\{1,\,\|C\|^{\beta-1}\}\|C^{-\beta}f_0\|\lambda_n^{\min\{1,\beta\}}
\end{align*}


\section{Some Properties of Power Series Kernels}
\label{Properties_PSKernels}

For the RKHS $\mathcal{H}_k$ endowed with the PS-kernel $k$, 
let us consider the derivative $\frac{\partial f_{\bxi}}{\partial\xi_i} $ of $f\in\mathcal{H}_k$.
We prove 
\begin{align}
 \label{eqn:deriv-inclusion}
\frac{\partial f_{\bxi}}{\partial\xi_i}\in\mathcal{H}_k,\quad i=1,\ldots,d,\ \ \bxi\in(0,1)^d
\end{align}
for $f\in\mathcal{H}_k$. We apply Theorem 7.17 of \cite{rudin06:_princ_mathem_analy}, which claims the
following assertion. 

\noindent
 {\bf Theorem}:
 \emph{
 Suppose $\{F_n\}$ is a sequence of functions, differentiable on $[a,b]$ and such that $F_n(\xi_0)$ converges
 for some point $\xi_0$ on $[a,b]$. If the sequence of derivatives $\{F_n'\}$ converges uniformly on $[a,b]$,
 then $\{F_n\}$ converges 
 uniformly on $[a,b]$, to a function $F$, and $\displaystyle{}F'=\lim_{n\rightarrow\infty}F_n'$. 
}
 \begin{proof}
  [Proof of \eqref{eqn:deriv-inclusion}]
  Suppose that the coefficients of $f\in\mathcal{H}_k$ are $c_{\balpha},\,\balpha\in\Nbb_0^d$.
  From the definition of the power series kernel, we have $\sum_{\balpha}|c_{\balpha}|<\infty$. 
  Indeed, the inequality is derived from $\sum_{\balpha}(\balpha!)^2 c_{\balpha}^2/w_{\balpha}<\infty$ and 
  $\sum_{\balpha}w_{\balpha}/(\balpha!)^2<\infty$. 
  For a fixed $\x\in\mathcal{Z}$ and $(\xi_2,\ldots,\xi_d)\in[0,1]^{d-1}$,
  let us define $F_n(\xi_1;\x)$ and $G_n(\xi_1;\x)$ as 
\begin{align*}
 F_n(\xi_1;\x)
 &=\sum_{\balpha:|\balpha|\leq n} c_\balpha  \bxi^{\balpha}\x^{\balpha}
  =\sum_{\balpha:|\balpha|\leq n} c_\balpha \xi_1^{\alpha_1}\xi_2^{\alpha_2}\cdots \xi_d^{\alpha_d}\x^{\balpha},   \\ 
 G_n(\xi_1;\x)
 &= \frac{\partial}{\partial\xi_1}F_n(\xi_1;\x)
 = \sum_{\balpha:|\balpha|\leq n} c_\balpha \alpha_1  \xi_1^{\alpha_1-1}\xi_2^{\alpha_2}\cdots \xi_d^{\alpha_d}\x^{\balpha}, 
\end{align*}
where $|\balpha|$ of $\balpha\in\Nbb_0^d$ denotes the sum of all elements in $\balpha$. 
 We define $G(\xi_1;\x)=\lim_{n\rightarrow\infty}G_n(\xi_1;\x)$.
 Then, we have $G(\xi_1;\cdot )\in\mathcal{H}_k$ as the function of $\x$ if $\xi_1<1$. 
  To prove that,
  note that there exists a constant $B>0$ such that $0\leq\alpha_1\xi_1^{\alpha_1-1}<B$ for all $\alpha_1\in\Nbb_0$ if $\xi_1\in[0,1)$. 
  Then, the inequality $|c_\balpha \alpha_1\xi_1^{\alpha_1-1} \xi_2^{\alpha_2} \cdots \xi_d^{\alpha_d}|\leq B|c_{\balpha}|$ 
  guarantees
  the convergence of $G_n$ and $G(\xi_1;\cdot )\in\mathcal{H}_k$. 
  Next, we prove $\frac{\partial}{\partial\xi_1}f_{\bxi}=G(\xi_1;\cdot )$. 
  Choose $\varepsilon>0$ such that the closed interval $I=[\xi_1-\varepsilon,\xi_1+\varepsilon]$ is included in
  the open interval $(0,1)$. 
  There exists a constant $B>0$ such that for any $\xi\in{I}$ and any $\alpha\in\Nbb$
  the inequality $0\leq\alpha\xi^{\alpha-1}<{B}$ holds. 
  These facts yield
\begin{align*}
 \sup_{\xi_1\in{I}}|G(\xi_1;\x)- G_n(\xi_1;\x)|
 &\leq
 \sup_{\xi_1\in{I}}\sum_{\balpha:|\balpha|> n} |c_\balpha \alpha_1  \xi_1^{\alpha_1-1}\xi_2^{\alpha_2}\cdots \xi_d^{\alpha_d}\x^{\balpha}| \\
 &\leq
 \sum_{\balpha:|\balpha|> n} B|c_\balpha| \longrightarrow0\ \ (n\rightarrow\infty). 
\end{align*}
Note that $F_n(\xi_1;\x)\rightarrow f_{\bxi}(\x)$ holds uniformly on $\xi_1\in{I}$ as $n\rightarrow\infty$. 
The above theorem in \cite{rudin06:_princ_mathem_analy} 
guarantees $\frac{\partial}{\partial\xi_1}f_{\bxi}(\x)=G(\xi_1;\x)$
for arbitrary $\x\in\mathcal{Z}$, when $0\leq \xi_1<1$. 
Eventually, $\frac{\partial}{\partial\xi_1}f_{\bxi}=G(\xi_1;\cdot)\in\mathcal{H}_k$ holds 
at $\bxi=(\xi_1,\ldots,\xi_d)$ with $\xi_1<1$. In general, \eqref{eqn:deriv-inclusion} holds. 
\end{proof} 

In terms of the derivative $\frac{\partial{f}_{\bxi}}{\partial\xi_i}$, we prove some formulae. 
 \begin{lem}
  \label{appendix:lemma:deriv-unif-conv}
 Let us consider the RKHS $\mathcal{H}_k$ endowed with the PS kernel $k$. 
 Suppose $\|\widehat{f}-f^*\|=o_P(1)$ for the kernel-based estimator $\widehat{f}$ of $f^*\in\mathcal{H}_k$.
 Then, we have 
 $\big\|\frac{\partial \widehat{f}_{\bxi}}{\partial\xi_i}-\frac{\partial f_{\bxi}^*}{\partial\xi_i}\big\|=o_P(1)$
 for $\bxi=(\xi_1,\ldots,\xi_d)\in(0,1)^d$.
 \end{lem}
\begin{proof}
 Suppose
 \begin{align*}
  f&=\sum_{\balpha\in\Nbb_0^d}c_{\balpha}\x^\balpha, \quad 
  \widehat{f}=\sum_{\balpha\in\Nbb_0^d}\widehat{c}_{\balpha}\x^\balpha. 
 \end{align*}
 As shown in the proof of \eqref{eqn:deriv-inclusion}, 
 the derivative of $f_{\bxi}$ is given 
 by the sum of the derivative of each term $c_{\balpha}\bxi^\balpha \x^\balpha$.
 Let $B$ be a positive number such that $0\leq\alpha_1 \xi_1^{\alpha_1-1}<B$ for all $\alpha_1\in\Nbb$. 
 The convergence of
 $\big\|\frac{\partial \widehat{f}_{\bxi}}{\partial\xi_1}-\frac{\partial f_{\bxi}^*}{\partial\xi_1}\big\|$ 
 is confirmed by 
  \begin{align*}
  \big\|\frac{\partial \widehat{f}_{\bxi}}{\partial\xi_1}-\frac{\partial f_{\bxi}^*}{\partial\xi_1}\big\|^2
  & =
  \sum_{\balpha}
  \frac{(\balpha!)^2}{w_{\balpha}}
  (\widehat{c}_{\balpha}-c_{\balpha})^2  (\alpha_1 \xi_1^{\alpha_1-1} \xi_2^{\alpha_2}\cdots\xi_d^{\alpha_d})^2 \\
  & \leq 
  B^2\sum_{\balpha}
  \frac{(\balpha!)^2}{w_{\balpha}}
  (\widehat{c}_{\balpha}-c_{\balpha})^2
  =B^2\|\widehat{f}-f\|^2 = o_P(1). 
  \end{align*}
 In general, we have
  $\big\|\frac{\partial \widehat{f}_{\bxi}}{\partial\xi_i}-\frac{\partial
 f_{\bxi}^*}{\partial\xi_i}\big\|=o_P(1)$. 
\end{proof}

  \begin{lem}
   \label{lemma:converge-diff-seq}
 Let us consider the RKHS $\mathcal{H}_k$ endowed with the PS kernel $k$. 
 We define $\e_i\in\Rbb^d$ as the unit vector with one in the $i$-th position and zeros otherwise.
 For $f\in\mathcal{H}_k$ and $\bxi\in(0,1)^d$, we have 
  \begin{align}
   \label{eqn:appendix-convf}
   &\lim_{t\rightarrow0}
   \big\|   f_{\bxi+t\e_i}-f_{\bxi} \big\| =0, \\
   \label{eqn:appendix-deriv-conv-RKHS}
   &\lim_{t\rightarrow0}
   \bigg\|
   \frac{f_{\bxi+t\e_i}-f_{\bxi}}{t}-\frac{\partial f_{\bxi}}{\partial\xi_i}
   \bigg\|
   =0
    \end{align}
  for $i=1,\ldots,d$. 
  \end{lem}
\begin{proof}
 We prove \eqref{eqn:appendix-deriv-conv-RKHS} for $i=1$.
 The equation \eqref{eqn:appendix-convf} is similarly proved. 
 Suppose that $f(\x)=\sum_{\balpha}c_{\balpha}\x^{\balpha}$
 with $\sum_{\balpha}\frac{(\balpha!)^2}{w_{\balpha}} c_{\balpha}^2<\infty$. 
 Note that both 
 $\frac{f_{\bxi+t\e_1}-f_{\bxi}}{t}$ and $\frac{\partial f_{\bxi}}{\partial\xi_1}$
 are included in $\mathcal{H}_k$. 
 Using the equality 
 \begin{align*}
  \frac{f_{\bxi+t\e_1}(\x)-f_{\bxi}(\x)}{t}-\frac{\partial f_{\bxi}}{\partial\xi_1}(\x)
  =
  \sum_{\balpha}c_{\balpha} 
  \left\{\frac{(\xi_1+t)^{\alpha_1}-\xi_1^{\alpha_1}}{t}-\alpha_1\xi_1^{\alpha_1-1}\right\}
  \xi_2^{\alpha_2}\cdots\xi_d^{\alpha_d}  \x^{\balpha}, 
 \end{align*}
 we have 
 \begin{align}
  \label{proof:appendix-deriv-conv-RKHS}
   \lim_{t\rightarrow0}
   \bigg\|
   \frac{f_{\bxi+t\e_1}-f_{\bxi}}{t}-\frac{\partial f_{\bxi}}{\partial\xi_1}
   \bigg\|^2
   \leq 
   \lim_{t\rightarrow0} \sum_{\balpha}\frac{(\balpha!)^2}{w_{\balpha}} c_{\balpha}^2
  \left\{\frac{(\xi_1+t)^{\alpha_1}-\xi_1^{\alpha_1}}{t}-\alpha_1\xi_1^{\alpha_1-1}\right\}^2. 
 \end{align}
Suppose that $0<\xi_1<1$ and  $0<\xi_1+t<1$ for $|t|\leq\varepsilon$, then  there exists a constant 
 $B_{\xi_1,\varepsilon}$ depending on $\xi_1$ and $\varepsilon$ such that
 for all $\alpha_1\in\Nbb$, 
\begin{align*}
 \left|\frac{(\xi_1+t)^{\alpha_1}-\xi_1^{\alpha_1}}{t}-\alpha_1\xi_1^{\alpha_1-1}\right|
& \leq
 \max_{|t|\leq\varepsilon}|\alpha_1(\xi_1+t)^{\alpha_1-1}-\alpha_1\xi_1^{\alpha_1-1}|  \\ 
& \leq
 \varepsilon\cdot \max_{|t|\leq\varepsilon}\alpha_1(\alpha_1-1)(\xi_1+t)^{\alpha_1-2} \\
& \leq
 \varepsilon\cdot \max_{|t|\leq\varepsilon}\alpha_1^2(\xi_1+t)^{\alpha_1-2} \\
 & \leq
 \varepsilon B_{\xi_1,\varepsilon}. 
\end{align*}
More precisely, $B_{\xi_1,\varepsilon}$ can be chosen as
\begin{align*}
 B_{\xi_1,\varepsilon}=4e^{-2}\max\big\{
 ((\xi_1+\varepsilon)\log(\xi_1+\varepsilon))^{-2},\
 ((\xi_1-\varepsilon)\log(\xi_1-\varepsilon))^{-2}
 \big\}. 
\end{align*}
Hence, for all $t$ such that $|t|\leq \varepsilon$ with a sufficiently small $\varepsilon$, we have 
\begin{align*}
  \sup_{t:|t|\leq \varepsilon}\sum_{\balpha}\frac{(\balpha!)^2}{w_{\balpha}} c_{\balpha}^2
  \left\{\frac{(\xi_1+t)^{\alpha_1}-\xi_1^{\alpha_1}}{t}-\alpha_1\xi_1^{\alpha_1-1}\right\}^2  
  \leq
  \varepsilon^2  B_{\xi_1,\varepsilon}^2\sum_{\balpha}\frac{(\balpha!)^2}{w_{\balpha}} c_{\balpha}^2
 \end{align*}
 Since $\displaystyle\lim_{\varepsilon\searrow0}\varepsilon^2  B_{\xi_1,\varepsilon}^2=0$ for a fixed $\xi_1\in(0,1)$, 
 the left-hand side of
 \eqref{proof:appendix-deriv-conv-RKHS}
 converges to zero. 
\end{proof}
 
 Let us calculate the derivative of the quadratic loss
 $\widehat{L}(\bxi;f)=\frac{1}{2}\<f_{\bxi},\widehat{C} f_{\bxi}\> - \<\widehat{g}, f_{\bxi}\>$.
 For $\bxi\in(0,1)^d$ and non-zero $t\in\Rbb$, we have 
 \begin{align}
  &\phantom{=}
  \frac{1}{t}\{\widehat{L}(\bxi+t\e_i;f)-\widehat{L}(\bxi;f)\} \nonumber\\
  \label{eqn:appendix:loss-limit-decomp}
  &=
  \frac{1}{2}\<f_{\bxi+t\e_i}, \widehat{C}\frac{f_{\bxi+t\e_i}-f_{\bxi}}{t}\>
  +
  \frac{1}{2}\<\frac{f_{\bxi+t\e_i}-f_{\bxi}}{t},
  \widehat{C}f_{\bxi}\>-\<\widehat{g},\frac{f_{\bxi+t\e_i}-f_{\bxi}}{t}\>  
 \end{align}
 We calculate the first term in the right-hand side of the above equation: 
\begin{align*}
 &\phantom{=}
 \bigg|
 \<f_{\bxi+t\e_i}, \widehat{C}\frac{f_{\bxi+t\e_i}-f_{\bxi}}{t}\>
 -\<f_{\bxi}, \widehat{C}\frac{\partial{f_{\bxi}}}{\partial\xi_i}\>
 \bigg|  \\
 &\leq
  \bigg|
  \<f_{\bxi+t\e_i},
 \widehat{C}\big(\frac{f_{\bxi+t\e_i}-f_{\bxi}}{t}-\frac{\partial{f_{\bxi}}}{\partial\xi_i}\big)\>
 \bigg|
 +
 \bigg|
 \<f_{\bxi+t\e_i}-f_{\bxi}, \widehat{C}\frac{\partial{f_{\bxi}}}{\partial\xi_i}\>
 \bigg|  \\
 &\leq 
 \|\widehat{C}\|\|f_{\bxi+t\e_i}\|
 \left\|   \frac{f_{\bxi+t\e_i}-f_{\bxi}}{t}-\frac{\partial{f}_{\bxi}}{\partial\xi_i}   \right\|
 +
 \|\widehat{C}\|\bigg\|\frac{\partial{f_{\bxi}}}{\partial\xi_i}\bigg\| \|f_{\bxi+t\e_i}-f_{\bxi}\|. 
\end{align*}
Lemma~\ref{lemma:converge-diff-seq} guarantees that the upper bound of the above equation converges to zero. 
Hence, the first term of \eqref{eqn:appendix:loss-limit-decomp} converges to
$\<f_{\bxi}, \widehat{C}\frac{\partial{f_{\bxi}}}{\partial\xi_i}\>$ as $t$ tends to zero.
Similar calculation yields that 
 \begin{align}
  \label{eqn:appendix:deriv-quad-empiricalloss}
 \frac{\partial}{\partial\xi_i}\widehat{L}(\bxi;f)
 =
 \frac{1}{2}\big\<f_{\bxi},\widehat{C}\frac{\partial{f_{\bxi}}}{\partial\xi_i}\big\>
 +
 \frac{1}{2}\big\<\frac{\partial{f_{\bxi}}}{\partial\xi_i},\widehat{C}f_{\bxi}\big\>
 -
 \big\<\widehat{g},\,\frac{\partial{f_{\bxi}}}{\partial\xi_i}\big\>
 \end{align}
 for $i=1,\ldots,d$, $f\in\mathcal{H}$ and $\bxi\in(0,1)^d$.
 
Similar formula holds for the derivatives of the expected quadratic loss, 
$\frac{\partial}{\partial\xi_i}L(\bxi;f)$. 

\section{Theoretical Results}
\label{appendix:proofs}

\subsection{Proof of Theorem~\ref{thm:var_consistency}}
\label{appendix:variable-consistency}
\begin{proof}[Proof of (i)]
    Define $M(\bm{\xi})$ and $M_n(\bm{\xi})$ for $\bm{\xi}\in[0,1]^d$ as 
  \begin{align*}
   M(\bm{\xi}):=L(\bm{\xi};f^*),\ \  \text{and}\ \ 
   M_n(\bm{\xi}):=\widehat{L}(\bm{\xi};\widehat{f})+\eta_n\1^T\bm{\xi}. 
  \end{align*}
 First, we prove $\sup_{\bm{\xi}\in[0,1]^d}|M(\bm{\xi})-M_n(\bm{\xi})|=o_P(1)$.
 The condition (b) in Assumption~\ref{assump:lossfunc}
 leads to 
  \begin{align*}
   \sup_{\bm{\xi}}|\widehat{L}(\bm{\xi};\widehat{f})-L(\bm{\xi};f^*)|
   \leq
   \sup_{\bm{\xi}}|\widehat{L}(\bm{\xi};\widehat{f})-\widehat{L}(\bm{\xi};f^*)|
   +
   \sup_{\bm{\xi}}|\widehat{L}(\bm{\xi};f^*)-L(\bm{\xi};f^*)|=o_p(1). 
  \end{align*}
   Then, we obtain
  \begin{align*}
   \sup_{\bm{\xi}\in[0,1]^d}|M(\bm{\xi})-M_n(\bm{\xi})|
   &=
   \sup_{\bm{\xi}\in[0,1]^d}|L(\bm{\xi};f^*)-\widehat{L}(\bm{\xi};\widehat{f})-\eta_n\1^T\bm{\xi}| \\
   &\leq
   \sup_{\bm{\xi}\in[0,1]^d}|L(\bm{\xi};f^*)-\widehat{L}(\bm{\xi};\widehat{f})|+d\eta_n=o_P(1). 
  \end{align*}
 We prove the consistency of $\widehat{\bm{\xi}}_1$ using Theorem 5.7 of \cite{vaart00:asympstatis}.  
   Since $M_n(\widehat{\bm{\xi}})\leq M_n(\bm{\xi}^*)$ and $M_n(\bm{\xi}^*)=M(\bm{\xi}^*)+o_P(1)$, 
   we have $M_n(\widehat{\bm{\xi}})\leq M(\bm{\xi}^*)+o_P(1)$. Since $f^*(\z)$ depends only on $z_1,\ldots,z_s$, we have 
   \begin{align*}
    M(\widehat{\bm{\xi}}_1,\0)-M({\bm{\xi}}_1^*,\0)
    &=
    M(\widehat{\bm{\xi}}_1,\widehat{\bm{\xi}}_0)-M({\bm{\xi}}_1^*,{\bm{\xi}}_0^*)  \\
   &\leq
   M(\widehat{\bm{\xi}}_1,\widehat{\bm{\xi}}_0)-M_n(\widehat{\bm{\xi}}_1,\widehat{\bm{\xi}}_0)+o_P(1)\\
   &\leq
   \sup_{\bm{\xi}}|M(\bm{\xi})-M_n(\bm{\xi})|+o_P(1)=o_P(1). 
   \end{align*}
 For any given $\varepsilon>0$, the condition (a) in Assumption~\ref{assump:lossfunc} 
 guarantees that there exists $\gamma>0$ such that $M(\bm{\xi}_1,\0)>M(\bm{\xi}_1^*,\0)+\gamma$ 
 for any $\bm{\xi}_1$ with $\|\bm{\xi}_1-\bm{\xi}_1^*\|\geq \varepsilon$.
 Thus, the event $\|\widehat{\bm{\xi}}_1-\bm{\xi}_1^*\|\geq \varepsilon$ is included in the event
 $M(\widehat{\bm{\xi}}_1,\0)>M(\bm{\xi}_1^*,\0)+\gamma$.
 Since the probability of the latter event converges to $0$, $\widehat{\bm{\xi}}_1$ converges to $\bm{\xi}_1^*$ in
 probability. 
\end{proof}

\begin{proof}[Proof of (ii)]
   We define
   and $\nabla$ as the gradient operator
   $\big(\frac{\partial}{\partial\xi_1},\ldots,\frac{\partial}{\partial\xi_d}\big)^T$. 
   For $i=s+1,\ldots,d$, Assumption~\ref{assump:loss_deriv} leads to 
   \begin{align*}
    &\phantom{=}\sup_{\bm{\xi}\in[0,1]^d}\big|\frac{\partial}{\partial\xi_i}\widehat{L}(\bm{\xi};\widehat{f})\big| \\
     &=
    \sup_{\bm{\xi}\in[0,1]^d}\big|
    \frac{\partial}{\partial\xi_i}\widehat{L}(\bm{\xi};\widehat{f})-\frac{\partial}{\partial\xi_i}L(\bm{\xi};f^*)
    \big|\\ 
     &\leq
    \sup_{\bm{\xi}\in[0,1]^d}
    \big|
    \frac{\partial}{\partial\xi_i}\widehat{L}(\bm{\xi};\widehat{f})-\frac{\partial}{\partial\xi_i}\widehat{L}(\bm{\xi};f^*)
    \big|
     +
    \sup_{\bm{\xi}\in[0,1]^d}
    \big|
    \frac{\partial}{\partial\xi_i}\widehat{L}(\bm{\xi};f^*)-\frac{\partial}{\partial\xi_i}L(\bm{\xi};f^*)
    \big|\\ 
     &=O_P(\delta_n+\delta_n'). 
   \end{align*}
   The first equality comes from the assumption that $f^*$ does not depend on $z_{s+1},\ldots,z_d$. 
   The estimator $\widehat{\bm{\xi}}=(\xi_1,\ldots,\xi_d)=(\widehat{\bm{\xi}}_1,\, \widehat{\bm{\xi}}_0)$
   should satisfy the local optimality condition, i.e., for any $\c\in[0,1]^d$, 
  \begin{align*}
   (\nabla{\widehat{L}}(\widehat{\bm{\xi}};\widehat{f}) +\eta_n\1)^T(\c-\widehat{\bm{\xi}})\geq 0
  \end{align*}
   holds. Hence, the inequality 
  \begin{align*}
   \big(\frac{\partial}{\partial\xi_i}{\widehat{L}}(\widehat{\bm{\xi}};\widehat{f}) +\eta_n\1\big)
   (c_i-\widehat{\xi}_i)\geq 0
  \end{align*}
  should hold for any $i=s+1,\ldots,d$ and any $c_i\in[0,1]$. 
  Here, we assume that an element of $\widehat{\bm{\xi}}_0=(\widehat{\xi}_{s+1},\ldots,\widehat{\xi}_{d})$, say 
   $\widehat{\xi}_d$, is strictly positive. Then, by setting $c_d=0$
  we have 
  \begin{align*}
   \frac{\partial}{\partial\xi_d}\widehat{L}(\widehat{\bm{\xi}};\widehat{f})+\eta_n\leq 0. 
  \end{align*}
   Since $\frac{\partial}{\partial\xi_d}\widehat{L}(\widehat{\bm{\xi}};\widehat{f})=O_P(\delta_n+\delta_n')$ 
   and the positive sequence $\eta_n$ dominates $\delta_n+\delta_n'$ by the assumption,  
   the above inequality leads to the contradiction as the sample size
  becomes large. 
  Therefore, $\widehat{\bm{\xi}}_0=\0$ holds with high probability for sufficiently large sample size. 
  \end{proof}
  
\subsection{Condition (a) in Assumption~\ref{Problem_Setup}}
\label{appendix:cond-a-proof}

\subsubsection{Kernel-ridge estimator}
Note that the function $f\in\mathcal{H}$ is differentiable if the kernel function is differentiable. 
We see that the minimizer of $L(f)$ is $f^*$ and that the expected loss function is expressed as 
\begin{align*}
 L(f)=\int(y-f^*(\x))^2 p(\x,y)d{\x}dy + \int(f^*(\x)-f(\x))^2 p(\x)d{\x}
\end{align*}
up to constant terms. 
Suppose that the random variable $\varepsilon_i$ is bounded and $p(\x)$ is positive on the input domain. 
Since the function $f\in\mathcal{H}$ is continuous, the second term of the above equation becomes
positive when the function $f\in\mathcal{H}$ is different from $f^*$.
Suppose that $f^*(x_1,\ldots,x_d)$ essentially depends only on $x_1,\ldots,x_s, s\leq d$. 
Then, we have
\begin{align*}
 L(f^*)=L((\bm{\xi}_1^*,\bm{\xi}_0^*);f^*), 
\end{align*}
where $\bm{\xi}_1^*=\1\in\Rbb^s,\,\bm{\xi}_0^*=\0\in\Rbb^{d-s}$.  
Suppose that $L(\bm{\xi};f^*)=L(f^*)$ holds for $\bm{\xi}=(\bm{\xi}_1,\0), {\bm\xi}_1=(\xi_1,\ldots,\xi_s)\in[0,1]^s$ such that
$\|{\bm\xi}_1-{\bm\xi}_1^*\|_2\geq \varepsilon$.
Since the optimal function is unique on $\mathcal{H}$, 
the equality $f^*=f_{\bm{\xi}}^*\in\mathcal{H}$ should hold. 
Without loss of generality, we assume $\xi_1<1, \xi_2=\ldots=\xi_s=1$. Then,
for any $(x_1,\ldots,x_s)$ in the domain, we have
\begin{align*}
 f^*(x_1,x_2,\ldots,x_s)=f^*(\xi_1x_1,x_2,\ldots,x_s)=f^*(\xi_1^2x_1,x_2,\ldots,x_s)\rightarrow
 f^*(0,x_2,\ldots,x_s). 
\end{align*}
Hence, $f^*$ does not depend on $x_1$. This contradicts the assumption of $f^*$.
Thus, we have $L((\bm{\xi}_1,\0);f^*)>L((\bm{\xi}_1^*,\0);f^*)$ for $\bm{\xi}_1\neq\bm{\xi}_1^*$. 
Moreover, if the function $L((\bm{\xi}_1,\0);f^*)$ is continuous w.r.t. $\bm{\xi}_1$,
the condition (a) in Assumption~\ref{Problem_Setup} holds, because the set 
$\{\bxi_1\in[0,1]^s\,:\,\|\bxi_1-\bxi_1^*\|_2\geq \varepsilon\}$ is a compact set. 
We can prove the continuity of $L(\bm{\xi};f^*)$ from
the boundedness of the random variable $\varepsilon_i$ and the Lebesgue's dominated convergence theorem. 


\subsubsection{Kernel-based density-ratio estimator}

The condition (a) in Assumption~\ref{Problem_Setup} is confirmed in the same way as the kernel-ridge
estimator. Hence, we omit the details. 


\subsubsection{Kernel-based density estimator}

For $f^*\in\mathcal{H}_k$, the probability density $p(\z)\propto\exp(f^*(\z))$ is strictly positive on the
compact domain. 
Suppose that essentially $f^*$ depends only on $x_1,\ldots,x_s, s\leq d$ and that 
$L(\bxi;f^*)=L(f^*)$ holds for $\bxi=(\bxi_1,0)\in\Rbb^d,],\bxi_1\in[0,1]^s$ such that $\|\bxi_1-\1\|_2\geq \varepsilon$. 
We have $\partial_a{f}_{\bxi}^*=\partial_{a}f^*,\,a=1,\ldots,d$ on the domain and hence, $f^*=f_{\bxi}^*+c,\, c\in\Rbb$ holds. 
Without loss of generality, we assume $\xi_1<1$ and $\xi_2=\cdots=\xi_s=1$.
Suppose that $c\neq0$. Since $\|f^*\|_\infty\leq \kappa\|f^*\|<\infty$, we have 
\begin{align*}
 f^*(x_1,\dots,x_s)
& =
 f^*(\xi_1x_1,x_2,\dots,x_s)+c \\
& =
 f^*(\xi_1^2x_1,x_2,\dots,x_s)+2c \\
& = 
 f^*(\xi_1^kx_1,x_2,\dots,x_s)+kc \rightarrow \sign(c)\times\infty\ \  (k\rightarrow\infty). 
\end{align*}
Thus, $c=0$ should hold. Again we have
\begin{align*}
 f^*(x_1,\dots,x_s)
 =
 f^*(\xi_1x_1,x_2,\dots,x_s)
 =
 f^*(\xi_1^2x_1,x_2,\dots,x_s)\rightarrow f^*(0,x_1,x_2,\dots,x_s). 
\end{align*}
This means that the function $f^*$ does not depend on $x_1$. This is the contradiction. 
Therefore, we have $L((\bm{\xi}_1,\0);f^*)>L(f^*)$ for $\bxi_1\in[0,1]^s$ if $\bm{\xi}_1\neq\1$. 
Moreover, if the function $L((\bm{\xi}_1,\0);f^*)$ is continuous w.r.t. $\bm{\xi}_1$,
the condition (a) in Assumption~\ref{Problem_Setup} holds, because the set
$\{\bxi_1\in[0,1]^s\,:\,\|\bxi_1-\bxi_1^*\|_2\geq \varepsilon\}$ is a compact set. 
We can prove the continuity of $L(\bm{\xi};f^*)$ from
the boundedness of the derivatives of $f$ and the Lebesgue's dominated convergence theorem. 


\subsection{Proof of Lemma~\ref{prop:conv-operator_Assump2}}

 As shown in Section~\ref{subsec:PowerSeriesKernels}, 
 the inequality $\|f_{\bm{\xi}}\|\leq \|f\|$ holds for $f\in\mathcal{H}$ and $\bm{\xi}\in[0,1]^d$. 
 The norm $\|\widehat{C}\|$ is stochastically bounded because $\|\widehat{C}\|\leq \|\widehat{C}-C\|+\|C\|$. 
 Also, $\|\widehat{g}\|$ and $\|\widehat{f}\|$ are stochastically bounded. 
  Then, we have 
\begin{align*}
 |\widehat{L}(\bm{\xi};\widehat{f})-\widehat{L}(\bm{\xi};f^*)|
 &\leq 
 \frac{1}{2}|\<\widehat{f}_{\bm{\xi}},\widehat{C}(\widehat{f}_{\bm{\xi}}-f^*_{\bm{\xi}})\>|
 +
 \frac{1}{2}|\<\widehat{f}_{\bm{\xi}}-f^*_{\mathcal{\bm\xi}},\widehat{C}f^*_{\bm{\xi}}\>|
 +|\<\widehat{g},\widehat{f}_{\bm{\xi}}-f^*_{\mathcal{\bm\xi}}\>|  \\
 &\leq
 \frac{1}{2}\|\widehat{f}_{\bm{\xi}}\|\|\widehat{C}\|\|\widehat{f}_{\bm{\xi}}-f^*_{\bm{\xi}}\|
 +
 \frac{1}{2}\|f_{\bm{\xi}}^*\| \|\widehat{C}\|\|\widehat{f}_{\bm{\xi}}-f^*_{\bm{\xi}}\|
 +\|\widehat{g}\|\|\widehat{f}_{\bm{\xi}}-f^*_{\mathcal{\bm\xi}}\| \\
 &\leq
  \frac{1}{2}\|\widehat{f}\|\|\widehat{C}\|\|\widehat{f}-f^*\|
 +
 \frac{1}{2}\|f^*\| \|\widehat{C}\| \|\widehat{f}-f^*\| +\|\widehat{g}\|\|\widehat{f}-f^*\|. 
\end{align*}
  The upper bound does not depend on $\bxi$ and converges to zero in probability
  due to Assumption~\ref{assump:uniform-consistency}. Hence, we have 
\begin{align*}
 \sup_{\bm{\xi}\in[0,1]^d}|\widehat{L}(\bm{\xi};\widehat{f})-\widehat{L}(\bm{\xi};f^*)| = o_P(1). 
\end{align*}
  Let us consider the difference $|\widehat{L}(\bm{\xi};f^*)-L(\bm{\xi};f^*)|$.
  Its supremum w.r.t $\bxi$ is bounded above by
\begin{align*}
 \sup_{\bxi\in[0,1]^d}|\widehat{L}(\bm{\xi};f^*)-L(\bm{\xi};f^*)|
 &=
 \sup_{\bxi\in[0,1]^d}\left|
 \frac{1}{2}\<f_{\bm{\xi}}^*,(\widehat{C}-C)f_{\bm{\xi}}^*\>-\<\widehat{g}-g,f_{\bm{\xi}}^*\>
 \right| \\
 &\leq
 \sup_{\bxi\in[0,1]^d}\left\{
 \frac{1}{2}\|f_{\bm{\xi}}^*\|^2\|\widehat{C}-C\|+\|f_{\bm{\xi}}^*\| \|\widehat{g}-g\|
 \right\}\\
 &\leq
 \frac{1}{2}\|f^*\|^2\|\widehat{C}-C\|+\|f^*\| \|\widehat{g}-g\|. 
\end{align*}
  Since $\|\widehat{C}-C\|$ and $\|\widehat{g}-g\|$ converge to zero as $n\rightarrow\infty$, 
  the uniform convergence condition (b) in Assumption~\ref{assump:lossfunc} holds.

\subsection{Proof of Lemma~\ref{prop:conv-operator_Assump3}}

For simplicity, we assume $\mathcal{I}=\{I\}$ and $\mathcal{J}=\{J\}$, i.e, both families contain only one
subset and that both $\|h_I\|_\infty$ and $\|\bar{h}_J\|_\infty$ are bounded above by $1$. 
Let us define $I=\{i_1,\ldots,i_a\}$ and $J=\{j_1,\ldots,j_b\}$. 
  For $I$ and $J$, let $\widetilde{I}=\{i\}\cup{I}$ and $\widetilde{J}=\{i\}\cup{J}$. 
  Then, for $f\in\mathcal{H}$, we have
    \begin{align*}
     \frac{\partial}{\partial\xi_i}\frac{1}{2}\<f_{\bxi},\widehat{C}f_{\bxi}\>
     &=
     \frac{1}{n}\sum_{\ell=1}^{n}
     h_I(\z_\ell)
     \frac{\partial}{\partial\xi_i}  \frac{1}{2} (\partial_{I}f(\bxi\circ\z_\ell) \xi_{i_1}\cdots\xi_{i_a})^2  \\
     &=
     \frac{1}{n}\sum_{\ell=1}^{n}
     h_I(\z_\ell)
     \partial_{I}f(\bxi\circ\z_\ell) \xi_{i_1}\cdots\xi_{i_a}\\
     &\qquad\qquad \times
     \left(
     \partial_{\widetilde{I}}f(\bxi\circ\z_\ell)z_{\ell,i}\xi_{i_1}\cdots\xi_{i_a}
    +
     \partial_{I}f(\bxi\circ\z_\ell)
    \frac{\partial}{\partial\xi_i}\xi_{i_1}\cdots\xi_{i_a}
     \right),  \\
     \frac{\partial}{\partial\xi_i}\<\widehat{g},f_{\bxi}\> 
     &=
     \frac{1}{n'}\sum_{\ell=1}^{n'}\bar{h}_J(\z_\ell')
         \frac{\partial}{\partial\xi_i} \partial_{J}f(\bxi\circ\z_{\ell}') \xi_{j_1}\cdots\xi_{j_b}   \\
     &=
     \frac{1}{n'}\sum_{\ell=1}^{n'}
     \bar{h}_J(\z_\ell')
     \bigg(
     \partial_{\widetilde{J}}f(\bxi\circ\z_\ell')z_{\ell,i}' \xi_{j_1}\cdots\xi_{j_b}
     +
     \partial_{J}f(\bxi\circ\z_\ell')    \frac{\partial}{\partial\xi_i}\xi_{j_1}\cdots\xi_{j_b} 
     \bigg), 
    \end{align*}
    for $\z_\ell=(z_{\ell,1},\ldots,z_{\ell,d})$ and $\z_\ell'=(z'_{\ell,1},\ldots,z'_{\ell,d})$. 
    Since $\z\in(-1,1)^d$ and $\bxi\in[0,1]^d$ for
    $\widehat{L}(\bxi;f)=\frac{1}{2}\<f_{\bxi},\widehat{C}f_{\bxi}\>-\<\widehat{g},f_{\bxi}\>$, we have  
     \begin{align*}
      &\phantom{=}
      \sup_{\bxi\in[0,1]^d}
      \left|\frac{\partial}{\partial\xi_i}\widehat{L}(\bxi;\widehat{f})-\frac{\partial}{\partial\xi_i}\widehat{L}(\bxi;f^*)\right| \\
      &\leq
      \|\partial_{I}\widehat{f} \partial_{\widetilde{I}}\widehat{f}-\partial_{I}f^* \partial_{\widetilde{I}} f^*\|_\infty
      +
      a\|(\partial_{I}\widehat{f})^2-(\partial_I{f^*})^2\|_\infty
      +\|\partial_{\widetilde{J}}\widehat{f} - \partial_{\widetilde{J}}f^*\|_\infty
      +b \|\partial_{J}\widehat{f}-\partial_{J}f^*\|_\infty\\
      &\leq
      (1+a)\kappa^2(\|\widehat{f}\|+\|f^*\|)\|\widehat{f}-f^*\|+(1+b)\kappa\|\widehat{f}-f^*\|=o_P(1), 
     \end{align*}
    where the inequality $\|\partial_{I}f\|_\infty\leq \kappa\|f\|$ was used. 
    Then, the condition (a) in Assumption~\ref{assump:loss_deriv} holds. 

    Let us consider the condition (b) in Assumption~\ref{assump:loss_deriv}. 
    Define $u(\z;\bxi)$ and $v(\z;\bxi)$ as
     \begin{align*}
      u(\z;\bxi)
      &=
     \partial_{I}f^*(\bxi\circ\z) \xi_{i_1}\cdots\xi_{i_a}
     \left(
      \partial_{\widetilde{I}}f^*(\bxi\circ\z)z_{i}\xi_{i_1}\cdots\xi_{i_a}
      +
      \partial_{I}f^*(\bxi\circ\z)
      \frac{\partial}{\partial\xi_i}\xi_{i_1}\cdots\xi_{i_a}
      \right),  \\
      v(\z;\bxi)
      &=
      \partial_{\widetilde{J}} f^*(\bxi\circ\z)z_i \xi_{j_1}\cdots\xi_{j_b}
      +
      \partial_{J}f^*(\bxi\circ\z)\frac{\partial}{\partial\xi_i}\xi_{j_1}\cdots\xi_{j_b}. 
     \end{align*}
    Then, we have
    \begin{align*}
     &\phantom{=}\frac{\partial}{\partial\xi_i}\widehat{L}(\bxi;f^*)     -\frac{\partial}{\partial\xi_i}L(\bxi;f^*) \\
     &=
     \frac{1}{n}\sum_{\ell=1}^{n}\left\{u(\z_\ell;\bxi)-\int u(\z;\bxi) p(\z)d\z\right\}
     +\frac{1}{n'}\sum_{\ell'=1}^{n'}\left\{v(\z_{\ell'};\bxi)-\int v(\z;\bxi) q(\z)d\z\right\}. 
    \end{align*}
    The uniform convergence property is related to the covering number of the function sets, 
    $\mathcal{U}=\{u(\z;\bxi)\,|\,\bxi\in[0,1]^d\}$ and $\mathcal{V}=\{v(\z;\bxi)\,|\,\bxi\in[0,1]^d\}$.
    From the inequalities such as $\|\partial_{I}f\|\leq \kappa\|f\|$, we find that the following inequalities holds: 
    \begin{align*}
     \sup_{\z\in\mathcal{Z},\bxi\in[0,1]^d}\bigg|\frac{\partial}{\partial{\xi}_j}u(\z;\bxi)\bigg|\leq \gamma \|f^*\|^2,\quad
     \sup_{\z\in\mathcal{Z},\bxi\in[0,1]^d}\bigg|\frac{\partial}{\partial{\xi}_j}v(\z;\bxi)\bigg|\leq \gamma \|f^*\|, 
    \end{align*}
  where $\gamma$ is a positive constant. 
    Let us define $N_p(\mathcal{F},r)$ be the converging number of the set $\mathcal{F}$ under the $p$-norm.
    Then, for $C_{\mathrm{Lip}}=\gamma \|f^*\|^2+\gamma \|f^*\|$, we have
    \begin{align*}
     \max\{N_\infty(\mathcal{U},r),   N_\infty(\mathcal{V},r)\}
     \leq N_2([0,1]^d,r/C_{\mathrm{Lip}})\leq
     \bigg(\frac{C_{\mathrm{Lip}}\sqrt{d}}{r}\bigg)^d. 
    \end{align*}
    Lemma 2.1 and 3.1 of \cite{Book:VanDeGeer:EmpiricalProcess} ensures the uniform law of large numbers 
    over $\mathcal{U}$ and $\mathcal{V}$.
    A simple calculation yields that the condition (b) in Assumption~\ref{assump:loss_deriv} 
    holds with the convergence rate $\delta_n'=\sqrt{\log{n}/n}$. 


\subsection{Proof of Lemma~\ref{prop:conv-operator_Assump3_add}}

 As shown in \eqref{eqn:deriv-inclusion}, we have $\frac{\partial{f_{\bxi}}}{\partial{\xi_i}}\in\mathcal{H}$ 
 for $f\in\mathcal{H}$ and $\bm{\xi}\in(0,1)^d$.
 Since 
 $\frac{\partial}{\partial\xi_i}\widehat{L}(\bxi;f)$ is continuous on $[0,1]^d$,
 the supremum of 
 $\sup_{\bm{\xi}\in[0,1]^d}|\frac{\partial}{\partial\xi_i}
 \widehat{L}(\bm{\xi};\widehat{f})-\frac{\partial}{\partial\xi_i} \widehat{L}(\bm{\xi};f^*)|$
 in Assumption~\ref{assump:loss_deriv} 
 can be replaced with the supremum on the open hypercube $(0,1)^d$. 
 Using \eqref{eqn:appendix:deriv-quad-empiricalloss}, 
 the derivative of the empirical quadratic loss is expressed as 
\begin{align*}
 \frac{\partial}{\partial\xi_i}\widehat{L}({\bm{\xi}};f)
 &=
 \frac{1}{2}\big\<\frac{\partial f_{\bm{\xi}}}{\partial\xi_i},\widehat{C}f_{\bm{\xi}}\big\>
 +
 \frac{1}{2}\big\<f_{\bm{\xi}},\widehat{C}\frac{\partial f_{\bm{\xi}}}{\partial\xi_i} \big\>
 -\big\<\widehat{g},\frac{\partial f_{\bm{\xi}}}{\partial\xi_i}\big\>
\end{align*}
 for $\bxi\in(0,1)^d$. 
Then, we have 
 \begin{align*}
  &\phantom{=}
  \sup_{\bxi\in(0,1)^d}
  \left|  \frac{\partial}{\partial\xi_i}\widehat{L}(\bxi;\widehat{f})-\frac{\partial}{\partial\xi_i}\widehat{L}(\bxi;f^*)\right| \\ 
  &\leq
  \sup_{\bxi\in(0,1)^d}\bigg\{
  \frac{1}{2}\big|  \big\<
  \frac{\partial \widehat{f}_{\bm{\xi}}}{\partial\xi_i}-\frac{\partial f^*_{\bm{\xi}}}{\partial\xi_i}, \widehat{C}\widehat{f}_{\bm{\xi}}
  \big\>  \big|
  +
  \frac{1}{2}\big|  \big\<\frac{\partial f^*_{\bm{\xi}}}{\partial\xi_i}, \widehat{C}(f^*_{\bm{\xi}}-\widehat{f}_{\bm{\xi}})\big\>\big|  \\
  &\phantom{\leq}\qquad \qquad \qquad 
  +\frac{1}{2}\big|\big\<
  \widehat{f}_{\bm{\xi}} - f^*_{\bm{\xi}},
  \widehat{C}\frac{\partial \widehat{f}_{\bm{\xi}}}{\partial\xi_i}
  \big\>\big|
  +
  \frac{1}{2}\big|
  \big\< f^*_{\bm{\xi}}, \widehat{C}(\frac{\partial \widehat{f}_{\bm{\xi}}}{\partial\xi_i}
  -\frac{\partial \widehat{f}_{\bm{\xi}}}{\partial\xi_i})\big\>\big|
   +
  \big|
  \<\widehat{g},\frac{\partial \widehat{f}_{\bm{\xi}}}{\partial\xi_i}-\frac{\partial f^*_{\bm{\xi}}}{\partial\xi_i} \>
  \big|
  \bigg\}  \\
  &\leq
  \left\{\kappa^2\beta (\|\widehat{f}\|+\|f^*\|)+\kappa\beta\right\}\|\widehat{f}-f^*\|=o_P(1), 
 \end{align*}
 where we used the conditions in Lemma~\ref{prop:conv-operator_Assump3_add} and the inequalities
 \begin{align*}
  \|f_{\bxi}\|_\infty\leq \kappa\|f_{\bxi}\|\leq \kappa\|f\|,\quad \bigg\|\frac{\partial{f}_{\bxi}}{\partial\xi_i}\bigg\|_\infty
  =
  \sup_{\z}\bigg|\frac{\partial{f}}{\partial z_i} (\bxi\circ\z)z_i\bigg|
  \leq
  \bigg\|\frac{\partial{f}}{\partial z_i}\bigg\|_\infty
  \leq \kappa\|f\|. 
 \end{align*}
 Assumption~\ref{assump:uniform-consistency} is also used to derive the stochastic order. 
   
 Next, let us consider the condition (b) in Assumption~\ref{assump:loss_deriv}. 
 As mentioned in the last part of Section~\ref{Properties_PSKernels}, 
 the derivative of the expected quadratic loss is expressed as 
\begin{align*}  
 \frac{\partial}{\partial\xi_i}L({\bm{\xi}};f^*)
=
 \frac{1}{2}\big\<\frac{\partial f^*_{\bm{\xi}}}{\partial\xi_i},Cf^*_{\bm{\xi}}\big\>
 +
 \frac{1}{2}\big\<f^*_{\bm{\xi}},C\frac{\partial f^*_{\bm{\xi}}}{\partial\xi_i} \big\>
 -\big\<g,\frac{\partial f^*_{\bm{\xi}}}{\partial\xi_i}\big\>
\end{align*}
for $\bxi\in(0,1)^d$. Then, we have 
 \begin{align*}
&\phantom{\leq}  \sup_{\bxi\in(0,1)^d}
  \left| 
 \frac{\partial}{\partial\xi_i}\widehat{L}({\bm{\xi}};f^*)
 -
 \frac{\partial}{\partial\xi_i}L({\bm{\xi}};f^*)  \right| \\
&  \leq
  \sup_{\bxi\in(0,1)^d}
  \left\{
  \frac{1}{2}\left|\left\<\frac{\partial f^*_{\bm{\xi}}}{\partial\xi_i},(\widehat{C}-C)f^*_{\bm{\xi}} \right\>\right|
 +
  \frac{1}{2}\left|
  \big\<f_{\bm{\xi}}^*,(\widehat{C}-C)\frac{\partial f^*_{\bm{\xi}}}{\partial\xi_i} \big\>
  \right|
  +\left|\big\<\widehat{g}-g,\frac{\partial f_{\bm{\xi}}}{\partial\xi_i}\big\>\right| \right\} \\
  & \leq
  \sup_{\bxi\in(0,1)^d}  \left\|  \frac{\partial f^*_{\bm{\xi}}}{\partial\xi_i}  \right\|
  \big(\|f^*\|    \|\widehat{C}-C\|  +   \|\widehat{g}-g\|\big), 
 \end{align*}
Since $\|\widehat{C}-C\|$ and $\|\widehat{g}-g\|$ converge to zero in probability as 
$n\rightarrow\infty$,
the second equation in Assumption~\ref{assump:loss_deriv} holds
due to the boundedness of $\sup_{\bxi\in(0,1)^d}  \left\|  \frac{\partial f^*_{\bm{\xi}}}{\partial\xi_i} \right\|$. 


\end{document}